\relax

\newif\ifarxiv
\arxivtrue

\newcommand{\forarxiv}[2]{\ifarxiv{}#1\else{}#2\fi}

\ifarxiv
\documentclass[a4paper,twocolumn,10pt]{article}
\usepackage[left=2cm,right=2cm,top=2cm,bottom=2cm]{geometry}
\else
\documentclass[letterpaper]{article}
\usepackage{aaai21}
\fi

\usepackage{times}
\usepackage{helvet}
\usepackage{courier}
\usepackage[hyphens]{url}
\usepackage{graphicx}
\usepackage{natbib}
\usepackage{caption}
\ifarxiv
\else
\fi
\usepackage[switch]{lineno}
\usepackage[utf8x]{inputenc}
\usepackage{booktabs}
\usepackage{amsfonts}
\usepackage{nicefrac}
\usepackage{microtype}

\usepackage{enumitem}

\usepackage{amsmath,amssymb,amsthm}
\usepackage[usenames,dvipsnames,svgnames,table]{xcolor}
\usepackage{wrapfig}
\usepackage{stmaryrd}

\usepackage{tikz}
\usetikzlibrary{trees}
\usepackage{forest}
\usetikzlibrary{matrix}

\usepackage{algorithm}
\usepackage{subcaption}
\usepackage{listings}

\usepackage{letltxmacro}
\newcommand*{\SavedLstInline}{}
\LetLtxMacro\SavedLstInline\lstinline
\DeclareRobustCommand*{\lstinline}{
  \ifmmode
    \let\SavedBGroup\bgroup
    \def\bgroup{
      \let\bgroup\SavedBGroup
      \hbox\bgroup
    }
  \fi
  \SavedLstInline
}
\newcommand{\code}[1]{\texttt{\lstinline[mathescape,classoffset=1,keywordstyle=\color{black},basicstyle=\color{black},classoffset=0,keywordstyle=\color{black}]{#1}}}

\usepackage{appendix}
\usepackage[capitalise]{cleveref}
\crefname{listing}{Algorithm}{Algorithms}
\Crefname{listing}{Algorithm}{Algorithms}

\theoremstyle{definition}
\newtheorem{theorem}{Theorem}
\newtheorem{lemma}[theorem]{Lemma}
\newtheorem{corollary}[theorem]{Corollary}

\newtheorem{assumptions}[theorem]{Assumptions}
\newtheorem{examplex}[theorem]{Example}
\newenvironment{example}
  {\pushQED{\qed}\examplex}
  {\popQED\endexamplex}

\definecolor{darkred}{rgb}{.7,0,0}
\definecolor{darkgreen}{rgb}{0,.5,0}
\definecolor{darkblue}{rgb}{0,0,.8}
\definecolor{darkcyan}{rgb}{0,0.6,.6}
\definecolor{darkorange}{rgb}{.8,.4,0}
\definecolor{gray}{rgb}{.4,.4,.4}

\iffalse
\newcommand{\todo}[1]{\textcolor{darkorange}{(\emph{TODO: #1})}}
\newcommand{\comment}[1]{\textcolor{darkblue}{(\emph{#1})}}
\newcommand{\warning}[1]{\textcolor{red}{(\emph{WARNING: #1})}}
\newcommand{\quest}[1]{\textcolor{darkgreen}{(\emph{Q: #1})}}

\else
\newcommand{\todo}[1]{}
\newcommand{\comment}[1]{}
\newcommand{\warning}[1]{}
\newcommand{\quest}[1]{}

\fi

\newcommand{\ie}{\emph{i.e.}, }
\newcommand{\eg}{\emph{e.g.}, }

\DeclareMathOperator*{\argmin}{\text{argmin}}

\newcommand{\D}{\text{d}}
\newcommand{\grad}{\nabla}

\newcounter{alphoversetcount}
\newcommand{\alphnextref}{\stepcounter{alphoversetcount}\text{(\alph{alphoversetcount})}}
\newcommand{\alphoverset}[1]{\overset{\alphnextref{}}{#1}}
\newcommand{\resetalph}{\setcounter{alphoversetcount}{0}}

\title{
Policy-Guided Heuristic Search with Guarantees
}
\author{
    Laurent Orseau,\textsuperscript{\rm 1}
    Levi H. S. Lelis\textsuperscript{\rm 2} \\
}
\ifarxiv
\date{
    \textsuperscript{\rm 1}DeepMind, UK \\
    \textsuperscript{\rm 2}Department of Computing Science, Alberta Machine Intelligence Institute (Amii), University of Alberta, Canada\\
    lorseau@google.com, levi.lelis@ualberta.ca
}
\else
\affiliations{
    \textsuperscript{\rm 1}DeepMind, UK \\
    \textsuperscript{\rm 2}Department of Computing Science, Alberta Machine Intelligence Institute (Amii), University of Alberta, Canada\\
    lorseau@google.com, levi.lelis@ualberta.ca

}
\fi

\newcommand{\node}{n}
\newcommand{\parent}{\text{par}}
\newcommand{\children}{\mathcal{C}}

\newcommand{\ancn}{\text{anc}_*}
\newcommand{\descn}{\text{desc}_*}
\newcommand{\goalset}{\nodeset_{\mathcal{G}}}
\newcommand{\rootnode}{\node_0}
\newcommand{\nodeset}{\mathcal{N}}
\newcommand{\leafset}{\mathcal{L}}

\newcommand{\stateset}{\mathcal{S}}

\newcommand{\pol}{\pi}

\newcommand{\levints}{\text{LevinTS}}
\newcommand{\ouralg}{\text{PHS}}
\newcommand{\ouralglong}{Policy-guided Heuristic Search}
\newcommand{\ouralgh}{\ouralg\ensuremath{_h}}
\newcommand{\ouralghat}{\ouralg*}

\newcommand{\flts}{\varphi}
\newcommand{\fltsmon}{\flts^{\scriptscriptstyle+}}
\newcommand{\jlts}{\eta}
\newcommand{\jltsmon}{\jlts^{\scriptscriptstyle+}}
\newcommand{\glts}{g}
\newcommand{\hlts}{h}
\newcommand{\llts}{\ell}
\newcommand{\Llts}{L}
\newcommand{\jltshat}{\hat{\jlts}}
\newcommand{\fltshat}{\hat{\flts}}

\newcommand{\state}{\code{state}}
\newcommand{\issolution}{\code{is_solution}}

\newcommand{\by}[1]{\ensuremath{{#1}{\times}{#1}}}

\ifarxiv
\else
\fi

\begin{document}
\nolinenumbers
\maketitle

\begin{abstract}
The use of a policy and a heuristic function for guiding search can be quite effective in adversarial problems, as demonstrated by AlphaGo and its successors, which are 
based on the PUCT search algorithm. 
While PUCT can also be used to solve single-agent deterministic problems, 
it lacks guarantees on its search effort and it can be computationally inefficient in practice.  
Combining the A* algorithm with a learned heuristic function 
tends to work better in these domains, but A* and its variants do not use a policy.
Moreover, the purpose of using A* is to find solutions of minimum cost, 
while we seek instead to minimize the \emph{search loss} (\eg the number of search steps).
\levints{} is guided by a policy and provides guarantees on the number of search steps that relate to the quality of the policy, but it does not make use of a heuristic function. In this work we introduce \ouralglong{} (\ouralg), a novel search algorithm that uses both a heuristic function and a policy and has theoretical guarantees on the search loss that relates to both the quality of the heuristic and of the policy.
We show empirically on the sliding-tile puzzle, Sokoban, and a puzzle from the commercial game `The Witness' that \ouralg{} enables the rapid learning of both a policy and a heuristic function and compares favorably with A*, Weighted A*, Greedy Best-First Search, \levints, and PUCT in terms of number of problems solved and search time in all three domains tested.
\end{abstract}

\comment{Renamings to be careful about:
\begin{itemize}
    \item instance $\rightarrow$ problem
    \item goal $\rightarrow$ solution (?)
    \item cost function $\rightarrow$ function
    \item cost $\rightarrow$ value
    \item state(n)=n to avoid re-expansions: Write an Assumption to use in the theorems?
    \item cut $\rightarrow$ pruning
\end{itemize}
}\comment{To do if accepted:
\begin{itemize}
    \item Say where to find the appendix (arXiv)
    \item Update Footnote 2 about datasets urls
    \item Make datasets available
    \item All formal results must have names
    \item The appendix is not published, hence we must add a link to the arxiv version
    (as said in the cfp)
    \item Give full PUCT pseudo-alg in appendix
    \item Say: for a fixed heuristic, the policy and the heuristic are independent of each other, 
    making learning easier/more stable: the update of the policy does not depend on the value of the heuristic and vice versa.
\end{itemize}}
\comment{Re PUCT: Our initial implementation was already more efficient than
DeepMind's \url{https://github.com/deepmind/acme/blob/0dfe551a54e6dc10cb0451ef4a683f5a0d1acd87/acme/agents/tf/mcts/search.py} by the use of a virtual loss (warning: we also said that we avoided repeated states. But it was buggy and happened at best during a 32 batch.).
However we realize after the deadline that this could be made even more efficient by avoiding by avoiding repeated states during the whole search, although it requires some guess work about what to do with repeating states on the same trajectory. Currently, we seeing a state for the second time (which is likely to result in loops) we stop the tree descent and backpropagate $\llts_{\max}$.
}
\warning{Does DeepCubeA use Q(s, a) or V(s)? If the former, then that may explain why WA* doesn't perform so well on Sokoban.}

\section{Introduction}

In this work\forarxiv{\footnote{This is an extended version of the paper accepted at AAAI 2021 with the same title.}}{\footnote{An extended version of this paper can be found on arXiv.}}
we are interested in tackling single-agent deterministic problems. 
This class of problems includes numerous real-world applications such as
robotics, planning and pathfinding, computational biology~\citep{edelkamp2010heuristic},
protein design~\citep{allouche2019protein}, and program synthesis~\citep{cropper2020programs}.

AlphaGo~\citep{silver2017mastering}, 
and descendants such as MuZero~\citep{schrittwieser2019mastering}
 combine a learned value function with a learned policy in the PUCT search algorithm~\citep{rosin2011pucb,kocsis2006uct}, which is a Monte-Carlo tree search algorithm~\citep{chang2005ams,coulom2007efficient}, 
to tackle 
stochastic and adversarial games with complete information,
and also a few single-agent games.
The policy guides the search locally by favoring the most promising children of a node, 
whereas the value function ranks paths globally, thus complementing each other. 
PUCT, based on UCT~\citep{kocsis2006uct} and PUCB~\citep{rosin2011pucb}, is indeed designed for adversarial and stochastic games such as Go and Chess,
and UCT and PUCB come with a guarantee that the value function converges to the true value function in the limit of exploration---however this guarantee may not hold when replacing actual rewards with an estimated value, as is done in the mentioned works.

Although these algorithms perform impressively well for some adversarial games,
it is not clear whether the level of generality of PUCT makes it the best fit for difficult deterministic single-agent problems
where a planning capability is necessary,
such as the PSPACE-hard Sokoban problem~\citep{Culberson1999}.
The more recent algorithm MuZero~\citep{schrittwieser2019mastering} adapts AlphaZero to single-agent Atari games, but these games are mostly reactive and MuZero performs poorly on games that require more planning like Montezuma's Revenge---although this may arguably pertain to MuZero needing to learn a model of the environment.

In the context of single-agent search the value function is known as a \emph{heuristic} function and it estimates the cost-to-go from a given state to a solution state. 
\citet{mcaleer2018solving} used MCTS to learn a heuristic function---but not a policy---to tackle the Rubik's cube,
but later replaced MCTS entirely with weighted A*~\citep{pohl1970heuristic,ebendt2009weighted},
which is a variant of the A* algorithm~\citep{hart1968aFormalBasis}
that trades off solution quality for search time.
They observe that ``MCTS has relatively long runtimes and often produces solutions many moves longer than the length of a shortest path''~\citep{agostinelli2019rubik}, and tackle a few more problems such as the sliding tile puzzle and Sokoban. 

Levin Tree Search (\levints)~\citep{orseau2018policy} uses a learned policy to guide 
its search in single-agent problems
and comes with an upper bound on the number of search steps that accounts for the quality of the policy.

In this work we combine the policy-guided search of the \levints{} algorithm with the heuristic-guided search of A* in an algorithm we call \ouralglong{} (\ouralg).
\ouralg{} retains the upper bound of \levints---we also prove an almost matching lower bound---but we extend this guarantee to the use of a heuristic function, showing that an accurate heuristic
can greatly reduce the number of search steps.

We compare our algorithm  with several policy-guided and heuristic search algorithms when learning
is interleaved with search in the Bootstrap process~\cite{jabbari2011bootstrap}: \levints{} uses a policy; A*, Weighted A* (WA*) and Greedy Best-First Search (GBFS)~\citep{doran1966gbfs} all use a heuristic, and PUCT uses both like \ouralg.
We evaluate these algorithms on the \by5 sliding-tile puzzle, Sokoban (Boxoban), and on a puzzle from the game `The Witness'. Our results show that \ouralg{} performs well on all three domains tested, while every other algorithm tested performs poorly in at least one of these domains.

\section{Notation and Background}\label{sec:notation}

Search algorithms solve single-agent problems by searching in the tree that defines the problem. 
Let $\nodeset$ be the set of all nodes that can be part of such a tree. 
The root of the tree is denoted $\rootnode\in\nodeset$.
For any node $n\in\nodeset$, 
its set of children is $\children(n)$,
its single parent is $\parent(n)$ (the root has no parent),
its set of ancestors including $n$ itself is $\ancn(n)$,
and its set of descendants including $n$ itself is $\descn(n)$;
its depth is $d(n)$ ($d(\rootnode)=0$) and also define $d_0(n) = d(n)+1$.

We say that a node is \emph{expanded} when a search algorithm 
generates the set of children of the node. 
All search algorithms we consider 
are constrained to expand
the root first and, subsequently, a node $n$ can be expanded only if its parent $\parent(n)$ has already been expanded.

A \text{problem} is defined by a root node $\rootnode$, 
a non-negative loss function $\llts:\nodeset\to[0, \infty]$, a set of solution nodes $\goalset\subseteq\descn(\rootnode)$,
and a \state{} function defined below.
There is a predicate \issolution$(n)$ available to the search algorithms
to test whether $n\in\goalset$, 
but it can be used on $n$ only after $\parent(n)$ has been expanded
(\ie $n$ must have been generated).
The loss $\llts(n)$ is incurred by the algorithm when expanding the node $n$.
For simplicity of the formalism, we assume that a node which is tested positive with \issolution{} is implicitly expanded---thus incurring a loss---but has no children.
The \emph{path} loss $\glts(n)$ of a node $n$ is the sum of the losses from the root to $n$, that is $\glts(n) = \sum_{n'\in\ancn(n)}\llts(n')$. For example, if the loss is one for any expanded node $n$, then $\glts(n) = d_0(n)$.
We assume that no infinite path has finite path loss.
For any search algorithm $S$, the \emph{search} loss $\Llts(S, n)$ is the 
sum of the individual losses $\llts(n')$ for all nodes $n'$ that have been expanded by the algorithm $S$, up to and including $n$,
and $\Llts(S, n)=\infty$ if $n$ is never expanded.
For example, if the loss $\llts(n)$ is the time needed to expand node $n$,
then $\Llts(S, n)$ corresponds to the computation time of the whole search when reaching $n$.

A policy $\pol:\nodeset\to[0, 1]$ is defined 
recursively for a child $n'$ of a node $n$:
$\pol(n')=\pol(n)\pol(n'|n)$
where the conditional probability
$\pol(n'|n)\in[0, 1]$ is such that $\sum_{n'\in\children(n)} \pol(n'|n)\leq1$,
and $\pol(\rootnode)=1$.
Therefore, $\pol(n)=\prod_{n'\in\ancn(n)\setminus\{\rootnode\}} \pol(n'|\parent(n'))$.
\warning{Need to remove the root from the product}

Let $\stateset\subseteq\nodeset$ be a set of `canonical nodes'.
The function $\state:\nodeset\to\stateset$ associates a node to a state (a canonical node),
with the constraints that
$\llts(n) = \llts(\state(n))$,
$\issolution(n) = \issolution(\state(n))$
and 
$\{\state(n'): n'\in\children(n)\} = \children(\state(n))$.
Search algorithms may use the \state{} function to avoid expanding nodes with the same states.

\subsection{Background}

The Best-First Search (BFS) search algorithm~\citep{Pearl84} (see \cref{alg:phs}) expands nodes by increasing value,
starting from the root and always expanding children only if their parent has been expanded already.
It does not expand nodes whose states have been visited before,
and returns the first solution node it expands.

The A* search algorithm~\citep{hart1968aFormalBasis} uses both the function $g$
and a heuristic function $h:\nodeset\to[0, \infty]$.
It uses BFS with the evaluation function $f(n)=g(n)+h(n)$.
If the heuristic $h$ is \emph{admissible}, 
\ie $g(n)+h(n)$ is a lower bound on the cost 
of the least-$g$-cost solution node below $n$, then A* is guaranteed
\footnote{Technically, this requires either that re-expansions
are performed or that the heuristic is \emph{consistent}.}
to return a solution with minimal $g$-cost.
Weighted A*~\citep{ebendt2009weighted} is a variant of A* that uses the evaluation function $f_w(n) = g(n) + w \cdot h(n)$ and has the guarantee that the first solution found has a $g$-cost no more than a factor $w$ of the minimum cost solution if $h$ is admissible and $w\geq 1$.

\levints{}~\citep{orseau2018policy} also uses the BFS algorithm, but with the evaluation function $f_\pol = d_0(n)/\pol(n)$ for a given policy $\pol$. 
\levints{} is guaranteed to expand no more than $d_0(n^*)/\pol(n^*)$ nodes until the first solution $n^*$ is found,
that is, with $\llts(\cdot)=1$ for all nodes, $\Llts(\text{\levints{}}, n^*) \leq d_0(n^*)/\pol(n^*)$.
Since $d_0(n^*)$ is fixed, it shows that the better the policy $\pol$, the shorter the search time. \Cref{thm:lower_bound} provides an almost-matching lower bound.

The PUCT algorithm~\citep{silver2016mastering} is not based on BFS,
but on UCT~\citep{kocsis2006uct},
which learns a value function from rewards,
and on PUCB~\citep{rosin2011pucb}, which additionally uses a policy prior.
Both ingredients are used to determine which node to expand next.
The PUCT formula depends on the current number of node expansions performed during search.
This dependency prevents the algorithm from being able to use a priority queue which requires the node values to not change over time.
Hence each time the algorithm performs a single node expansion, it has to go back to the root
to potentially take a different path for the next expansion.
In some cases, this additional cost can lead to a quadratic increase of the computation time compared to the number of node expansions.
Although this search strategy can be effective in stochastic and adversarial environments, it is often wasteful for deterministic single-agent problems.
The original UCT algorithm~\citep{kocsis2006uct} has regret-style guarantees, 
but as noted by \citet{orseau2018policy}, these guarantees are rarely meaningful 
in deterministic single-agent problems where rewards are often 0 until a solution is found; moreover, these guarantees do not hold for modern implementations that replace the rewards and the rollouts with a learned value function.

\subsection{Definition of the Search Problem}

Our overarching objective is to design algorithms that, 
when given a set of unknown tasks, 
solve all of them as quickly as possible
while starting with little to no knowledge about the tasks.
That is, for $K$ tasks, we want to devise an algorithm $S$ that minimizes the total \emph{search} loss $\sum_{k\in[K]} \min_{n^*\in{\goalset}_k}\Llts_k(S, n^*)$.
Observe that this departs from the 
more traditional objective of finding solutions of minimum \emph{path} loss for all tasks, that is, of minimizing $\sum_{k\in[K]} \glts(n^*_k)$
where $n^*_k$ is the solution node found by the algorithm for task $k$.
Hence, we do not require the solutions encountered by the search algorithms to be path-cost optimal. 

To this end we embed our algorithms into the Bootstrap process~\citep{jabbari2011bootstrap}, 
which iteratively runs a search algorithm with a bound on the number of node expansions (or running time) on a set of training tasks. The problems solved in a given iteration are used to update the parameters of a model encoding a heuristic function and/or a policy. The process is then repeated with the newly trained model, possibly after increasing the budget. This process does not use or generate a curriculum, but assumes that there are problems of varying difficulties.

\begin{algorithm}[htb!]
\begin{lstlisting}
def PHS($\rootnode$): return BFS($\rootnode$, $\flts$)

def BFS($\rootnode$, evaluate):
  q = priority_queue(order_by=evaluate)
  q.insert($\rootnode$)
  visited_states = {}
  while q is not empty:
    n = q.extract_min() # node of min value
    s = state(n)
    if s in visited_states:
      continue  # pruning
    visited_states += {s}
    incur_loss $\llts($n$)$
    if is_solution(n):
      return n
    # Node expansion
    for n$'$ in children(n):
      q.insert(n$'$, evaluate(n$'$))
  return False
\end{lstlisting}
\caption{The \ouralglong{} algorithm (\ouralg),
based on the Best-First Search algorithm (BFS).
(Re-expansions are not performed.)}
\label{alg:phs}
\end{algorithm}

\section{\ouralglong{}}

For all our theoretical results, we assume that $\forall n\in\nodeset:$ \state($n$)=$n$ to avoid having to deal with re-expansions\forarxiv{, but see \cref{sec:safepruning} for a more general treatment.}{.}

We generalize \levints{} first by considering arbitrary non-negative losses $\llts(n)$ rather than enforcing $\llts(n) = 1$, and by introducing a \emph{heuristic factor} $\jlts(n)\geq1$.
Our algorithm \ouralglong{} (\ouralg)
simply calls BFS (see \cref{alg:phs})
with the evaluation function $\flts$ defined as:
\begin{align*}
    \flts(n) &= \jlts(n)\frac{\glts(n)}{\pol(n)}\,,
\end{align*}
where $\glts$ and $\pol$ were defined in \cref{sec:notation},
and $\flts(n) = \infty$ if $\pol(n)=0$; \citet{orseau2018policy} discuss
how to prevent $\pol(n)=0$.
The factor $\glts(n)/\pol(n)$ can be viewed as an approximation
of the search loss $\Llts(\ouralg, n)$ when reaching $n$:
$\glts(n)$ is the loss of the path from the root to $n$, and $1/\pol(n)$
plays a similar role to an importance sampling weight
that rescales the (local) path loss to the (global) search loss.
Note that paths with large probability $\pol$ are preferred to be expanded before paths with small $\pol$-values.
The purpose of the new heuristic factor $\jlts(n)$
is to rescale the current estimate of the search loss $\Llts(\cdot, n)$
to an estimate of the search loss $\Llts(\cdot, n^*)$ at the least-cost solution node $n^*$ that descends from $n$.
This is similar to how $f(n)$ is an estimate of $g(n^*)$ in A*.
If both $\jlts(\cdot)$ and $\llts(\cdot)$ are 1 everywhere, \ouralg{} reduces to \levints{}.

Although for \levints{} $f_\pol$ is monotone non-decreasing from parent to child~\citep{orseau2018policy},
this may not be the case anymore for $\flts$ due to $\jlts$.
Thus, for the sake of the analysis we define the monotone non-decreasing evaluation function
$\fltsmon(n) = \max_{n'\in\ancn(n)}\flts(n')$.
Since in BFS nodes are expanded in increasing order of their value,
and a node $n$ cannot be expanded before any of its ancestors,
it is guaranteed in \ouralg{} that a node $n$ is expanded \emph{before} any other node $n'$ with $\fltsmon(n') > \fltsmon(n)$.

Define $\jltsmon(n) = \fltsmon(n)\frac{\pol(n)}{\glts(n)} \geq \jlts(n)$, and $\jltsmon(n)=\jlts(n)$ if $\glts(n)=0$,
such that
\begin{align*}
\fltsmon(n) = \jltsmon(n)\frac{\glts(n)}{\pol(n)}\,.
\end{align*}
Note that even if we ensure that $\jlts(n^*) = 1$ for any solution node $n^*\in\goalset$,
in general $\fltsmon(n^*)$ may still be larger than $\flts(n^*)$,
that is, we may have $\jltsmon(n^*) > 1$.

We now state our main result and explain it below.
Define the set of nodes
$\nodeset_\flts(n) = \{n'\in\descn(\rootnode):\fltsmon(n') \leq \fltsmon(n)\}$ 
of value at most $\fltsmon(n)$
and its set of leaves $\leafset_\flts(n) = \{n'\in\nodeset_\flts(n): \children(n')\cap\nodeset_\flts(n) = \emptyset \}$
that is, the set of nodes in $\nodeset_\flts(n)$ that do not have any children in this set.

\begin{theorem}[\ouralg{} upper bound]\label{thm:upper_bound}
For any non-negative loss function $\llts$,
for any set of solution nodes $\goalset\subseteq\descn(\rootnode)$, 
and for any given policy $\pol$ and any given heuristic factor $\jlts(\cdot) \geq 1$,
\ouralg{} returns a solution node $n^*\in\argmin_{n^*\in\goalset} \fltsmon(n^*)$
and the search loss is bounded by
\begin{align}\label{eq:upperboundsolution}
\Llts(\ouralg, n^*)  \leq \frac{\glts(n^*)}{\pol(n^*)}\jltsmon(n^*)\sum_{n\in\leafset_\flts(n^*)} \frac{\pol(n)}{\jltsmon(n)}\,.
\end{align}
\end{theorem}
\begin{proof}
Slightly abusing notation, for a set of nodes $\nodeset'$, define $\Llts(\nodeset') = \sum_{n\in\nodeset'} \llts(n)$
to be the cumulative loss over the nodes in $\nodeset'$.
Since $\fltsmon$ is non-decreasing from parent to child, $\nodeset_\flts(n)$ forms a tree rooted in $\rootnode$ and therefore all the nodes in $\nodeset_\flts(n)$ are expanded by BFS($\rootnode$, $\flts$)
before any other node not in $\nodeset_\flts(n)$.
Therefore, $\Llts(\ouralg, n) \leq \Llts(\nodeset_\flts(n))$.
Then,
\begin{align*}
\Llts(\nodeset_\flts(n))&= L\Big(\bigcup_{n'\in\leafset_\flts(n)} \ancn(n')\Big)
\leq \sum_{n'\in\leafset_\flts(n)} \Llts(\ancn(n')) \\
&=\sum_{n'\in\leafset_\flts(n)} \glts(n')
\leq \fltsmon(n)\sum_{n'\in\leafset_\flts(n)} \frac{\pol(n')}{\jltsmon(n')}
\end{align*}
where the last inequality follows from $\jltsmon(n')\frac{\glts(n')}{\pol(n')}=\fltsmon(n')\leq \fltsmon(n)$.
Finally, since this is true for any $n$, the result holds for the returned
$n^*\in\goalset$.
\end{proof}

We can derive a first simpler result:

\begin{corollary}[\ouralg{} upper bound with no heuristic]\label{cor:goverpol}
From \cref{thm:upper_bound}, if furthermore $\forall n, \jlts(n)=1$ then
\begin{align}\label{eq:goverpol}
\Llts(\ouralg, n^*) \leq \frac{\glts(n^*)}{\pol(n^*)}\,.
\end{align}
\end{corollary}
\begin{proof}
Follows from \cref{thm:upper_bound} and $\sum_{n'\in\leafset_\flts(n)}\pol(n')\leq 1$ for all $n$~\citep[Theorem 3]{orseau2018policy}.
\end{proof}

The bound in \cref{eq:goverpol} is similar to the bound provided for \levints{}
(and equal when $\llts(n)=1$ everywhere)
and shows the effect of the quality of the policy on the cumulative loss during the search: Since necessarily $\Llts(\ouralg, n^*) \geq \glts(n^*)$, 
the bound says that the search becomes more efficient as $\pol(n^*)$ gets closer to 1.
Conversely, with an uninformed policy $\pol$ the probability decreases exponentially with the depth, which means that the search becomes essentially a breadth-first search.
The bound in \cref{eq:upperboundsolution} furthermore shows the effect of the heuristic function where the additional factor can be interpreted as the ratio of the heuristic value $\jltsmon(n^*)$ at the solution to the (preferably larger) average heuristic value at the leaves of the search tree when reaching $n^*$.
If the heuristic is good, then this bound can substantially improve upon \cref{eq:goverpol}---but it can also degrade with a poor heuristic.

\begin{example}\label{ex:phs_bound}
Consider a binary tree where the single solution node $n^*$ is placed at random at depth $d$.
Assume that $\pol(n'|n)=1/2$ for both children $n'$ of any node $n$.
Assume that $\jlts(n)=\infty$ for all nodes except $\jlts(n)=1$ for the nodes on the path from the root to $n^*$, which makes \ouralg{} expand only the $d+1$ nodes on the path from the root to $n^*$.
Take $\llts(n)=1$ for all nodes.
Then \cref{eq:goverpol} tells us that the search loss (which is here the number of expanded nodes) is bounded by $(d+1)2^{d+1}$, which is correct but rather loose.
By taking the (very informative) heuristic information into account,
we have $\sum_{n\in\leafset_\flts(n^*)} \frac{\pol(n)}{\jltsmon(n)}=\pol(n^*)/\jlts(n^*)$
\comment{This is true not because $\jlts(n)=\infty$, but because $\leafset_\fltsmon(n^*) = \{n^*\}$.
Indeed, the nodes with $\jlts(n)$ do not belong to the $\nodeset_\fltsmon(n^*)$ because their
$\fltsmon(n) \geq \fltsmon(n*)$. However, for their parent $\fltsmon(\parent(n)) \leq \flts(n^*)$, but they are not leaves! Only $n^*$ is a leaf.}
and thus \cref{eq:upperboundsolution} tells us that the search loss is bounded
by $\glts(n^*)=d+1$, which is tight.
\end{example}
\todo{Another example where the heuristic value is taken into account}

The following lower bound is close to the bound of \cref{eq:goverpol} in general,
and applies to any algorithm, whether it makes use of the policy or not.
First, we say that a policy $\pol$ is \emph{proper} if for all nodes $n$, $\sum_{n'\in\children(n)}\pol(n'|n)=1$.

\begin{theorem}[Lower bound]\label{thm:lower_bound}
For every proper policy $\pol$ and non-negative loss function $\llts$ such that $\llts(\rootnode)=0$,
for every search algorithm $S$---that first expands the root and subsequently expands
only children of nodes that have been expanded---there are trees rooted in $\rootnode$
and sets of solution nodes $\goalset$ such that
\begin{align*}
    \min_{n^*\in\goalset}
    \Llts(S, n^*) \geq \frac{\glts(\hat{n}^*)}{\pol(n^*)},
    \quad\text{with } \hat{n}^* = \parent(\parent(n^*))\,.
\end{align*}
\end{theorem}
\begin{proof}
Consider the following infinite tree:
The root $\rootnode$ has $m$ children, $n_{1, 1}, n_{2, 1} \ldots n_{m, 1}$
and each child $n_{i, 1}$ is assigned an arbitrary conditional probability $\pol(n_{i, 1}|\rootnode)\geq 0$
such that $\sum_{i\in[m]}\pol(n_{i, 1}|\rootnode) = 1$.
Each node $n_{i, j}$ has a single child $n_{i, j+1}$ with $\pol(n_{i, j+1}|n_{i, j})=1$,
and is assigned an arbitrary loss $\llts(n_{i, j}) \geq 0$.
Observe that $\pol(n_{i, j}) = \pol(n_{i, 1}|\rootnode)$.

Now let the given search algorithm expand nodes in any order, as long as parents are always
expanded before their children,
until at least one node is expanded in each branch with positive probability---if this is not met then the bound holds trivially.
Then stop the search after any finite number of steps.
Let $\hat{n}_i$ be the last expanded node in each branch $i\in[m]$.
Then pick the node $\hat{n}^*$ among the $\hat{n}_i$ with smallest $\glts(\hat{n}_i)/\pol(\hat{n}_i)$.
Since this node $\hat{n}^*$ has been expanded, its unique child $\hat{n}'$ may have already been tested for solution, but not its grand-child since $\hat{n}'$ has not yet been expanded.
So we set $n^*$ to be the unique child of $\hat{n}'$, and set $\goalset=\{n^*\}$.
For each branch $i$ we have
$\glts(\hat{n}_i)/\pol(\hat{n}_i) \geq \glts(\hat{n}^*)/\pol(\hat{n}^*)$.
Therefore, recalling that the policy is proper,
the cumulative loss before testing if $n^*$ is a solution is at least
\begin{align*}
    \sum_{i\in[m]} g(\hat{n}_i) \geq \sum_{i\in[m]} \pol(\hat{n}_i) \frac{\glts(\hat{n}^*)}{\pol(\hat{n}^*)} = 
    \frac{\glts(\hat{n}^*)}{\pol(\hat{n}^*)} =
    \frac{\glts(\hat{n}^*)}{\pol(n^*)}\,.
    &\qedhere
\end{align*}
\end{proof}

\subsection{Admissible Heuristics}

We say that $\jlts$ is \emph{\ouralg-admissible} (by similarity to admissibility of heuristic functions for A*)
if for all nodes $n$, for all solution nodes $n^*$ below $n$
(\ie $n^*\in\descn(n)\cap\goalset$), we have 
$\flts(n) \leq \flts(n^*)$ and $\jlts(n^*)=1$, that is, $\flts(n^*) = \glts(n^*)/\pol(n^*)$.
This means that $\flts(n)$ always underestimates the cost of any descendant solution.
This ensures that for solution nodes, $\fltsmon(n^*) =\flts(n^*) = \glts(n^*)/\pol(n^*)$.
Note that we may still not have $\fltsmon(n) =\flts(n)$ for non-solution nodes.
Also observe that taking $\jlts(n) = 1$ for all $n$ is admissible, but not informed, similarly to $h(n)=0$ in A*.

Hence, \emph{ideally}, if $\jlts(n)= \infty$ when $\descn(n)\cap\goalset=\emptyset$,
\begin{align*}
\text{and }\jlts(n) &= \min_{n*\in\descn(n)\cap\goalset}\frac{\glts(n^*)/\pol(n^*)}{\glts(n)/\pol(n)}, \\
\text{then }\flts(n) &= \min_{n*\in\descn(n)\cap\goalset}\glts(n^*)/\pol(n^*)
\end{align*}
and thus, similarly to A*, \ouralg{} does not expand any node which does not lead to a solution node of minimal $\flts$-value.
When the heuristic is \ouralg-admissible but not necessarily ideal, we provide a refined bound of \cref{eq:upperboundsolution}:

\begin{corollary}[Admissible upper bound]\label{cor:boundoptimist}
If $\jlts$ is \ouralg-admissible, then the cumulative loss incurred by \ouralg{} before returning a solution node $n^*\in\argmin_{n^*\in\goalset}\fltsmon(n^*)$ is upper bounded by
\begin{align}\label{eq:admissiblebound}
\Llts(\ouralg, n^*)
\leq \frac{\glts(n^*)}{\pol(n^*)}\underbrace{\sum_{n\in\leafset_\flts(n^*)} \frac{\pol(n)}{\jltsmon(n)}}_{\Sigma}\,.
\end{align}
\end{corollary}
\begin{proof}
Follows from \cref{eq:upperboundsolution} with $\fltsmon(n^*)=\flts(n^*)$
due to the \ouralg-admissibility of $\jlts$.
\end{proof}

\cref{cor:boundoptimist} offers a better insight into the utility of a heuristic factor $\jlts$, in particular when it is \ouralg-admissible.
First, observe that $\sum_{n'\in\leafset_\flts(n)} \pol(n') \leq 1$, which means that the sum term
$\Sigma$ can be interpreted as an average.
Second, since $\jltsmon(\cdot) \geq 1$, then necessarily $\Sigma \leq 1$,
which means that \cref{cor:boundoptimist} is a strict improvement over \cref{eq:goverpol}.
$\Sigma$ can thus be read as the average search reduction factor at the leaves of the search tree
when finding node $n^*$.
\cref{cor:boundoptimist} shows that using a \ouralg-admissible heuristic factor $\jlts$ can help the search, on top of the help that can be obtained by using a good policy.
In light of this, we can now interpret $\jltsmon(n^*)$ in \cref{eq:upperboundsolution} as
an \emph{excess estimate} for inadmissible $\jlts$, and a trade-off appears between this excess and the potential gain in the $\Sigma$ term by the heuristic.
Interestingly, this trade-off disappears when the heuristic factor is \ouralg-admissible, that is, the bounds suggest that using a \ouralg-admissible heuristic is essentially `free.'

A number of traditional problems have readily available admissible heuristics for A*.
We show that these can be used in \ouralg{} too to define a \ouralg-admissible $\jlts$.
Let $h$ be a heuristic for A*. 
Making the dependency on $h$ explicit, define
\begin{align*}
    \jlts_h(n) = \frac{\glts(n)+h(n)}{\glts(n)}\,, \quad
    \text{then } \flts_h(n) = \frac{\glts(n) + h(n)}{\pol(n)}\,.
\end{align*}
\begin{theorem}[A*-admissible to \ouralg-admissible]
If $h$ is an admissible heuristic for A*, then $\jlts_h$ is \ouralg-admissible.
\end{theorem}
\begin{proof}
Since $h$ is admissible for A*, we have $\glts(n) + h(n) \leq \glts(n^*)$ for any solution node $n^*$ descending from $n$.
Hence $\flts_h(n) \leq \glts(n^*)/\pol(n) \leq \glts(n^*)/\pol(n^*) = \flts_h(n^*)$
and thus $\jlts_h$ is \ouralg-admissible.
\end{proof}

\newcommand{\hltsmon}{\hlts^+}
Define $\hltsmon(n) = \pol(n) \max_{n'\in\ancn(n)} \flts_h(n') - \glts(n)\geq \hlts(n)$
which is such that $\fltsmon_h(n) = \max_{n'\in\ancn(n)}\flts_h(n')$  is monotone non-decreasing. Then $\jltsmon_h(n) = 1+\hltsmon(n)/\glts(n)$.

\begin{corollary}[Admissible upper bound with $h$]\label{cor:boundoptimisth}
Given a heuristic function $h$,
if $\jlts=\jlts_h$ is \ouralg-admissible, then the cumulative loss incurred before 
returning a solution node $n^*\in\argmin_{n^*\in\goalset} \fltsmon(n^*)$ is upper bounded by
\begin{align*}
\Llts(\ouralg, n^*)
\leq \frac{\glts(n^*)}{\pol(n^*)}\sum_{n\in\leafset_\flts(n^*)}  \frac{\pol(n)}{1+\hltsmon(n)/\glts(n)}\,.
\end{align*}
\end{corollary}
\begin{proof}
Follows by \cref{cor:boundoptimist} and the definition of $\jltsmon_h$.
\end{proof}

\cref{cor:boundoptimisth} shows that the larger $h(n)$ (while remaining admissible), 
the smaller the sum, and thus the smaller the bound.
A well tuned heuristic can help reduce the cumulative loss by a large factor compared to the policy alone.

We call \ouralgh{} the variant of \ouralg{} that takes an A*-like heuristic function $h$
and uses $\jlts=\jlts_h$.

\subsection{A More Aggressive Use of Heuristics}

Similarly to A*, the purpose of the heuristic factor is to estimate the $g$-cost of the least-cost descendant solution node.
But even when $h$ is admissible and accurate, $\jlts_h$ is often not an accurate estimate of the cumulative loss at $n^*$ due to missing the ratio $\pol(n)/\pol(n^*)$---and quite often $\pol(n^*) \ll \pol(n)$.
Hence if $h$ is the only known heuristic information, we
propose an estimate $\jltshat_t$ that should be substantially more accurate 
if conditional probabilities and losses are sufficiently regular on the path to the solution,
by approximating $\pol(n^*)$  as $[\pol(n)^{1/\glts(n)}]^{\glts(n) + h(n)}$:
Take $\glts(n)=d(n)$ for intuition, 
then $\pol(n)^{1/\glts(n)}=p$ is roughly the average conditional probability
along the path from the root to $n$,
and thus $p^{\glts(n) + h(n)}$ is an estimate of the probability at depth $d(n^*)$,
that is, an estimate of $\pol(n^*)$.
This gives
\begin{align}
    \jltshat_h(n) = \frac{1+h(n)/\glts(n)}{\pol(n)^{h(n)/\glts(n)}}\,,\ 
    \fltshat_h(n) = \frac{\glts(n)+h(n)}{\pol(n)^{1+h(n)/\glts(n)}}
    \label{eq:jltshat}
\end{align}
so that $\fltshat_h(n)$ is an estimate of $\glts(n^*)/\pol(n^*)$.
The drawback of $\jltshat_h$ is that it may not be \ouralg-admissible anymore,
so while \cref{thm:upper_bound} still applies, \cref{cor:boundoptimist} does not.

We call \ouralghat{} the variant of \ouralg{} that defines
$\jlts$ as in \cref{eq:jltshat} based on some given (supposedly approximately admissible)
heuristic function $h$.

\section{Learning the Policy}\label{sec:polgrad}

We consider $K$ tasks, and quantities indexed by $k\in[K]$ have the same meaning as before, but for task $k$.
We assume that the policy $\pol_\theta$ is differentiable w.r.t. its parameters $\theta\in\Theta$.
Ideally, we want to optimize the policy so as to minimize the total search loss, \ie
the optimal parameters of the policy are
\begin{align*}
    \theta^* = \argmin_{\theta\in\Theta} \sum_{k\in[K]} \min_{n^*\in{\goalset}_k}
    \Llts_k(\ouralg_\theta, n^*)\,.
\end{align*}
Unfortunately, even if we assume the existence of a differentiable close approximation $\tilde{\Llts}_k$ of $\Llts_k$,
the gradient of the sum with respect to $\theta$ can usually not be obtained.
Instead the upper bound in \cref{eq:upperboundsolution} can be used as a surrogate loss function.
However, it is not ideally suited for optimization
due to the dependence on the set of leaves $\leafset_\flts(n^*)$.
Hence we make a crude simplification of \cref{eq:upperboundsolution} and assume that for task $k$,
$\Llts_k(\ouralg_\theta, n^*_k) \approx c_k \glts(n^*_k)/\pol_\theta(n^*_k)$
for some a priori unknown constant $c_k$ assumed independent of the parameters $\theta$
of the policy: indeed
the gradient of the term in \cref{eq:upperboundsolution} that $c_k$ replaces should usually be small
since $\pol_\theta(n)$ is often an exponentially small quantity with the search depth
and thus $\grad_\theta \pol_\theta(n)\approx 0$.
Then, since $\D f(x)/\D x = f(x) \D \log f(x)/\D x$,
\begin{align}\label{eq:grad}
    \grad_\theta \tilde{\Llts}_k(\ouralg_\theta, n^*_k) &
    = \tilde{\Llts}_k(\ouralg, n^*_k) \grad_\theta \log \tilde{\Llts}_k(\ouralg_\theta, n^*_k) \notag\\
    &\approx \tilde{\Llts}_k(\ouralg_\theta, n^*_k)\grad_\theta \log \Big(c_k \frac{\glts(n^*_k)}{\pol_\theta(n^*_k)}\Big) \notag\\
    &\approx \Llts_k(\ouralg_\theta, n^*_k)\grad_\theta \log \frac{1}{\pol_\theta(n^*_k)}\,.
\end{align}
Note that $\Llts_k(\ouralg_\theta, n^*_k)$ can be calculated during the search,
and that the gradient does not depend explicitly on the heuristic function $h$,
which is learned separately.
In the experiments, we use this form to optimize the policy for \ouralg{} and \levints{}.

\section{Experiments}

We use $\llts(\cdot)=1$ everywhere, so the search loss $\Llts$
corresponds to the number of node expansions.
We test the algorithms A*, GBFS, WA* (w=1.5), PUCT (c=1), 
\ouralgh{} (using $\jlts=\jlts_h$) and \ouralghat{} (using $\jlts=\jltshat_h$),
 and \levints{}.
Each algorithm uses one neural network to model the policy and/or the heuristic
to be trained on the problems it manages to solve\forarxiv{ (architectures in \cref{sec:nnarchs})}{}.
For PUCT,
we normalize the Q-values~\citep{schrittwieser2019mastering},
and use a virtual loss~\citep{chaslot2008parallel} of unit increment with the following selection rule for a child $n'$ of a node $n$,
\begin{align*}
    &\bar{h}(n') = (h(n') + \text{virtual\_loss}(n') - h_{\min}) / (h_{\max} - h_{\min}) \\
    &\text{PUCT}(n'; n) = \bar{h}(n) - c \pol(n'|n)\frac{\sqrt{\sum_{n''\in\children(n)} N(n'')}}{1+N(n')}
\end{align*}
where $N(n)$ is the number of times the node $n$ has been visited,
$c$ is a constant set to 1 in the experiments,
and keep in mind that $h$ corresponds to losses---hence the negative sign.
The node $n'$ with minimum PUCT value is selected.
The virtual loss allows for evaluating the nodes in batch: we first collect 32 nodes for evaluation using the PUCT rule and only then evaluate all nodes in batch with the neural network. The virtual loss allows us to sample different paths of the MCTS tree to be evaluated.
Similarly, for the BFS-based algorithms, 32 nodes are evaluated in batch with the neural network before insertion in the priority queue~\citep{agostinelli2019rubik}.
This batching speeds up the search for all algorithms.

Since we want to assess the cooperation of search and learning capabilities of the different algorithms without using domain-specific knowledge, our experiments are not directly comparable with domain-specific solvers (see \citet{Pereira2016b} for Sokoban).
In particular, no intermediate reward is provided, by contrast to \citet{orseau2018policy} for example.

\subsection{Domains}

\paragraph{Sokoban (\by{10})}
Sokoban is a PSPACE-hard grid-world puzzle where the player controls an avatar who pushes boxes to particular spots (boxes cannot be pulled).
We use the first 50\,000 problems training problems and the provided 1\,000 test problems from Boxoban~\citep{boxobanlevels}.

\paragraph{The Witness (\by4)}
The Witness domain is a NP-complete puzzle extracted from the video game of the same name~\citep{abel2020witness}
and consists in finding a path on a 2D grid that separates cells of different colors.
\footnote{Datasets and code are at \url{https://github.com/levilelis/h-levin}.} 

\paragraph{Sliding Tile Puzzle (\by5)}
The sliding tile puzzle is a traditional benchmark in the heuristic search literature
where heuristic functions can be very effective~\citep{korf1985depth,felner2004pdb}.
The training set is generated with random walks from the \by5 solution state 
with walks of lengths between 50 and 1\,000 steps---this is a difficult training set
as there are no very easy problems.
Test problems are generated randomly and unsolvable problems are filtered out by a parity check; note that this is like scrambling infinitely often, which makes the test problems often harder than the training ones.

\subsection{Training and Testing}\label{sec:traintest}

Each search algorithm start with a uniform policy and/or an uninformed heuristic
and follow a variant of the Bootstrap process~\citep{ernandes2004heuristics,jabbari2011bootstrap}:
All problems are attempted with an initial search step budget (2\,000 for Witness and Sokoban, 7\,000 for the sliding tile puzzle)
and the search algorithm may solve some of these problems.
After 32 attempted problems, a parameter update pass of the models is performed,
using as many data points as the lengths of the solutions of the solved problems among the 32.
If no new problem has been solved at the end of the whole Bootstrap iteration,
the budget is doubled.
Then the next iteration begins with all problems again.
The process terminates when the total budget of 7 days is spent (wall time, no GPU). 
We use the mean squared error (MSE) loss for learning the heuristic functions: 
For a found solution node $n^*$ 
the loss for node $n\in\ancn(n^*)$ is $[(d(n^*)-d(n))-h(n)]^2$,
where $h(n)$ is the output of the network.
Note that this loss function \emph{tends} to make the heuristic admissible.
The cross-entropy loss is used to learn the policy for PUCT.
The policy for \levints{}  and our \ouralg{} variants are based on the approximate gradient of the $\tilde{\Llts}_k$ loss (see \cref{sec:polgrad}).

Testing follows the same process as training (with a total computation time of 2 days), except that parameters are not updated, and the budget is doubled unconditionally at each new Bootstrap iteration.
The test problems are not used during training.
Each algorithm is trained 5 times with random initialization of the networks.
For fairness, we test each algorithm using its trained network 
that allowed the search algorithm to solve the largest number of training problems.

\begin{figure*}[tb!]
    \centering
    \includegraphics[width=0.344\textwidth]{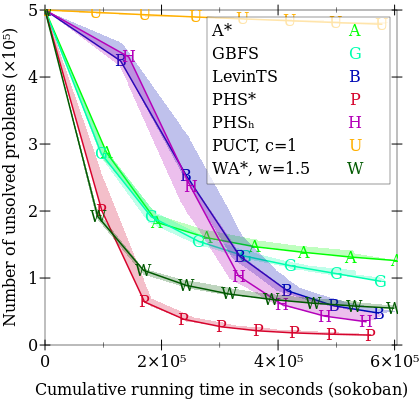}
    \includegraphics[width=0.323\textwidth]{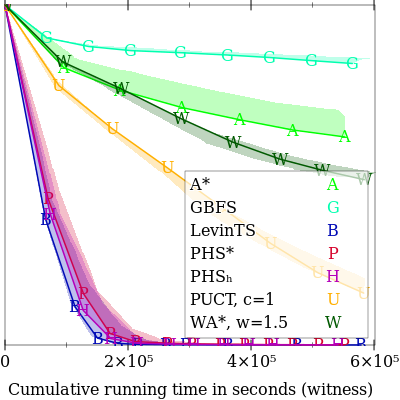}
    \includegraphics[width=0.323\textwidth]{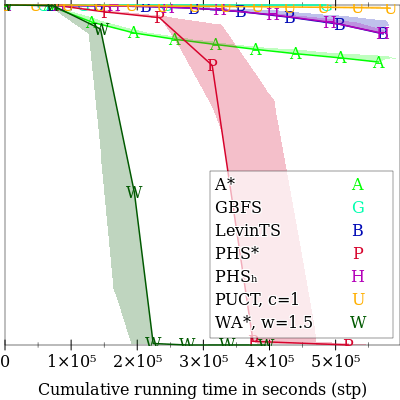}
    \caption{Learning curves using the Bootstrap process.
    Each point is a Bootstrap iteration.
    Lines correspond to the training runs with the least remaining unsolved problems at the end of the run.
    All 5 training runs per algorithm lie within the colored areas.}
    \label{fig:learning}
\end{figure*}

\subsection{Results}

The learning curves and test results are in \cref{fig:learning} and \cref{tab:test_results}\forarxiv{ (more information and results can be found in \cref{sec:learning_curves,sec:best_test_results})}{}.

\vspace{-1em}
\paragraph{Sokoban}
\ouralgh{} and \ouralghat{} obtain the best results on this domain,
showing the effectiveness of combining a policy and a heuristic function with BFS's efficient use of the priority queue.
\levints{} and WA* follow closely. 
A* takes significantly more time but finds shorter solutions, as expected.
By contrast to the other algorithms, GBFS is entirely dependent on the quality of the heuristic function and cannot compensate an average-quality heuristic with a complete search (A* and WA* also use the $g$-cost to help drive the search). 
After the training period, PUCT could not yet learn a good policy and a good heuristic from scratch: PUCT is designed for the larger class of stochastic and adversarial domains and takes a large amount of time just to expand a single new node.
Some variants of PUCT commit to an action after a fixed number of search steps (\eg \citet{racaniere2017imagination});
although once a good value function has been learned this may help with search time,
it also makes the search incomplete, which is not suitable when combining search and learning from scratch.
\citet{agostinelli2019rubik} report better results than ours for WA* (about 1\,050\,000 expansions in total on the test set), but their network has more than 30 times as many parameters, and they use 18 times as many training problems,
while also using a backward model to generate a curriculum.
For \levints{}, \citet{orseau2018policy} reported 5\,000\,000 expansions, using a larger network,
10 times as many training steps and intermediate rewards.
Our improved results for \levints{} are likely due to the cooperation of search and learning during training, allowing to gather gradients for harder problems.

\paragraph{The Witness}
This domain appears to be difficult for learning a good heuristic function,
with the policy-guided BFS-based algorithms being the clear winners. Even for \ouralg, the heuristic function does not seem to help compared to \levints{};
it does not hurt much either though, showing the robustness of \ouralg.
PUCT and WA* learned slowly during training but 
still manage to solve (almost) all test problems, at the price of many expansions.
GBFS performs poorly because it relies exclusively on the quality of the heuristic---by contrast, A* and WA* use the $g$-cost to perform a systematic search when the heuristic is not informative.

\paragraph{Sliding Tile Puzzle}
Only WA* and \ouralghat{} manage to solve all test problems.
\ouralgh{} seems to not be using the heuristic function to its full extent, by contrast to \ouralghat{}.
The trend in the learning curves (see \cref{fig:learning}\forarxiv{ and also \cref{sec:best_test_results}}{}) suggests that with more training \levints{}, \ouralgh{}, and A* would achieve better results, but both GBFS and PUCT seem to be stuck. 
Possibly the training set is too difficult for these algorithms and adding easier problems or an initial heuristic 
could help.

\newcommand{\badnum}[1]{\emph{#1}}
\newcommand{\atleast}[1]{\badnum{#1}}
\newcommand{\winner}[1]{\textbf{#1}}

\begin{table}[tb!]
\begin{center}
\begin{tabular}{lrrrr}
\toprule
\multicolumn{1}{c}{Alg.}& \multicolumn{1}{c}{Solved}  & \multicolumn{1}{c}{Length} & \multicolumn{1}{c}{Expansions} & \multicolumn{1}{c}{Time (s)} \\   
\midrule
\multicolumn{5}{c}{Sokoban \by{10} (test)} \\   
\midrule
PUCT, c=1  &     229 &      \badnum{24.9} &    \atleast{10\,021.3} & \atleast{39.7} \\
GBFS       &     914 &      \badnum{36.4} &     \atleast{5\,040.0} & \atleast{34.0} \\
A*         &     995 &      \winner{32.7} &     \atleast{8\,696.8} & \atleast{61.7} \\
WA*, w=1.5 &    \winner{1\,000} &      34.5 &     3\,729.1 & 25.5 \\
\levints{}      &    \winner{1\,000} &      40.1 &     2\,640.4 & 19.5 \\
\ouralgh{}        &    \winner{1\,000} &      39.1 &     2\,130.4 & 18.6 \\
\ouralghat{}       &    \winner{1\,000} &      37.6 &     \winner{1\,522.1} & \winner{11.3} \\
\midrule
\multicolumn{5}{c}{The Witness \by4 (test)} \\   
\midrule
GBFS       &     290 &      \badnum{13.3} &    \atleast{10\,127.9} & \atleast{44.6} \\
A*         &     878 &      \badnum{13.6} &     \atleast{9\,022.3} & \atleast{53.9} \\
WA*, w=1.5 &     999 &      \winner{14.6} &    \atleast{18\,345.2} & \atleast{71.5} \\
PUCT, c=1  &    \winner{1\,000} &      15.4 &     4\,212.1 & 23.6 \\
\ouralghat{}       &    \winner{1\,000} &      15.0 &      781.5 &  5.4 \\
\levints{}      &    \winner{1\,000} &      14.8 &      520.2 &  3.2 \\
\ouralgh{}        &    \winner{1\,000} &      15.0 &      \winner{408.1} &  \winner{3.0} \\
\midrule
\multicolumn{5}{c}{Sliding Tile Puzzle \by5 (test)} \\   
\midrule
GBFS       &       0 &        ---&         ---&   ---\\
PUCT, c=1  &       0 &        ---&         ---&   ---\\
A*         &       3 &      \badnum{87.3} &    \atleast{34\,146.3} & \atleast{27.2} \\
\ouralgh{}        &       4 &     \badnum{119.5} &    \atleast{58\,692.0} & \atleast{55.3} \\
\levints{}      &       9 &     \badnum{145.1} &    \atleast{39\,005.6} & \atleast{31.1} \\
\ouralghat{}       &    \winner{1\,000} &     224.0 &     2\,867.2 &  2.8 \\
WA*, w=1.5 &    \winner{1\,000} &     \winner{129.8} &     \winner{1\,989.8} &  \winner{1.6} \\
\bottomrule
\end{tabular}
\end{center}
\caption{Results on the tests sets.
Lengths are the solution depths, and the last three columns are averaged over solved problems only; hence numbers such as $\badnum{123}$ cannot be properly compared with.}
\label{tab:test_results}
\end{table}

\section{Conclusion}

We proposed a new algorithm called \ouralg{} that extends the policy-based \levints{} to using general non-negative loss functions and a heuristic function. 
We provided theoretical results relating the \emph{search} loss with the quality of both the policy and the heuristic.
If the provided heuristic function is \ouralg-admissible, a strictly better upper bound can be shown for \ouralg{} than for \levints{}.
In particular, an admissible heuristic for A* can be turned into a \ouralg-admissible heuristic,
leading to the variant \ouralgh.
We also provided a lower bound based on the information carried by the policy, 
that applies to any search algorithm.
The more aggressive variant \ouralghat{} is the only algorithm which consistently solves all test problems in the three domains tested.
It would be useful to derive more specific bounds showing when \ouralghat{} is expected to work strictly better than \ouralgh{}.
In this paper, the learned heuristic corresponds to the distance to the solution as for A*, but it may be better to directly learn the heuristic $\jlts$ to estimate the actual \emph{search} loss at the solution.
This may however introduce some difficulties since the heuristic function to learn would now depend on the policy, making learning possibly less stable.

\subsubsection*{Acknowledgements}
We would like to thank 
Tor Lattimore,
Marc Lanctot,
Michael Bowling,
Ankit Anand,
Théophane Weber,
Joel Veness
and the AAAI reviewers for their feedback and helpful comments. This research was enabled by Compute Canada (www.computecanada.ca) and partially funded by Canada's CIFAR AI Chairs program. 

\ifarxiv
\bibliographystyle{unsrtnat}
\fi
\bibliography{biblio}

\begin{thebibliography}{31}
\providecommand{\natexlab}[1]{#1}
\providecommand{\url}[1]{\texttt{#1}}
\expandafter\ifx\csname urlstyle\endcsname\relax
  \providecommand{\doi}[1]{doi: #1}\else
  \providecommand{\doi}{doi: \begingroup \urlstyle{rm}\Url}\fi

\bibitem[Edelkamp et~al.(2010)Edelkamp, Schroedl, and
  Koenig]{edelkamp2010heuristic}
Stefan Edelkamp, Stefan Schroedl, and Sven Koenig.
\newblock \emph{Heuristic Search: Theory and Applications}.
\newblock Morgan Kaufmann Publishers Inc., San Francisco, CA, USA, 2010.
\newblock ISBN 0123725127.

\bibitem[Allouche et~al.(2019)Allouche, Barbe, De~Givry, Katsirelos, Lebbah,
  Loudni, Ouali, Schiex, Simoncini, and Zytnicki]{allouche2019protein}
David Allouche, Sophie Barbe, Simon De~Givry, George Katsirelos, Yahia Lebbah,
  Samir Loudni, Abdelkader Ouali, Thomas Schiex, David Simoncini, and Matthias
  Zytnicki.
\newblock {Cost Function Networks to Solve Large Computational Protein Design
  Problems}.
\newblock In \emph{{Operations Research and Simulation in healthcare}}.
  {Springer}, 2019.

\bibitem[Cropper and Dumancic(2020)]{cropper2020programs}
Andrew Cropper and Sebastijan Dumancic.
\newblock Learning large logic programs by going beyond entailment.
\newblock In \emph{Proceedings of the Twenty-Ninth International Joint
  Conference on Artificial Intelligence, {IJCAI} 2020}, pages 2073--2079.
  ijcai.org, 2020.

\bibitem[Silver et~al.(2017)Silver, Hubert, Schrittwieser, Antonoglou, Lai,
  Guez, Lanctot, Sifre, Kumaran, Graepel, Lillicrap, Simonyan, and
  Hassabis]{silver2017mastering}
David Silver, Thomas Hubert, Julian Schrittwieser, Ioannis Antonoglou, Matthew
  Lai, Arthur Guez, Marc Lanctot, Laurent Sifre, Dharshan Kumaran, Thore
  Graepel, Timothy~P. Lillicrap, Karen Simonyan, and Demis Hassabis.
\newblock Mastering chess and shogi by self-play with a general reinforcement
  learning algorithm.
\newblock \emph{CoRR}, abs/1712.01815, 2017.

\bibitem[Schrittwieser et~al.(2020)Schrittwieser, Antonoglou, Hubert, Simonyan,
  Sifre, Schmitt, Guez, Lockhart, Hassabis, Graepel, Lillicrap, and
  Silver]{schrittwieser2019mastering}
Julian Schrittwieser, Ioannis Antonoglou, Thomas Hubert, Karen Simonyan,
  Laurent Sifre, Simon Schmitt, Arthur Guez, Edward Lockhart, Demis Hassabis,
  Thore Graepel, Timothy Lillicrap, and David Silver.
\newblock Mastering atari, go, chess and shogi by planning with a learned
  model.
\newblock \emph{Nature}, 588\penalty0 (7839):\penalty0 604--609, December 2020.

\bibitem[Rosin(2011)]{rosin2011pucb}
Christopher~D. Rosin.
\newblock Multi-armed bandits with episode context.
\newblock \emph{Annals of Mathematics and Artificial Intelligence}, 61\penalty0
  (3):\penalty0 203–230, March 2011.

\bibitem[Kocsis and Szepesv{\'a}ri(2006)]{kocsis2006uct}
Levente Kocsis and Csaba Szepesv{\'a}ri.
\newblock Bandit based monte-carlo planning.
\newblock In \emph{ECML}, pages 282--293. Springer Berlin Heidelberg, 2006.

\bibitem[Chang et~al.(2005)Chang, Fu, Hu, and Marcus]{chang2005ams}
Hyeong~Soo Chang, Michael~C. Fu, Jiaqiao Hu, and Steven~I. Marcus.
\newblock An adaptive sampling algorithm for solving markov decision processes.
\newblock \emph{Operations Research}, 53\penalty0 (1):\penalty0 126--139, 2005.

\bibitem[Coulom(2007)]{coulom2007efficient}
R{\'e}mi Coulom.
\newblock Efficient selectivity and backup operators in monte-carlo tree
  search.
\newblock In \emph{Computers and Games}, pages 72--83. Springer Berlin
  Heidelberg, 2007.

\bibitem[Culberson(1999)]{Culberson1999}
Joseph~C. Culberson.
\newblock {S}okoban is {PSPACE}-{C}omplete.
\newblock In \emph{Fun With Algorithms}, pages 65--76, 1999.

\bibitem[McAleer et~al.(2019)McAleer, Agostinelli, Shmakov, and
  Baldi]{mcaleer2018solving}
Stephen McAleer, Forest Agostinelli, Alexander Shmakov, and Pierre Baldi.
\newblock Solving the rubik's cube with approximate policy iteration.
\newblock In \emph{International Conference on Learning Representations
  (ICRL)}, 2019.

\bibitem[Pohl(1970)]{pohl1970heuristic}
Ira Pohl.
\newblock Heuristic search viewed as path finding in a graph.
\newblock \emph{Artificial Intelligence}, 1\penalty0 (3):\penalty0 193 -- 204,
  1970.

\bibitem[Ebendt and Drechsler(2009)]{ebendt2009weighted}
R{\"u}diger Ebendt and Rolf Drechsler.
\newblock Weighted a∗ search – unifying view and application.
\newblock \emph{Artificial Intelligence}, 173\penalty0 (14):\penalty0 1310 --
  1342, 2009.

\bibitem[Hart et~al.(1968)Hart, Nilsson, and Raphael]{hart1968aFormalBasis}
P.~E. Hart, N.~J. Nilsson, and B.~Raphael.
\newblock A formal basis for the heuristic determination of minimum cost paths.
\newblock \emph{IEEE Transactions on Systems Science and Cybernetics},
  SSC-4(2):\penalty0 100--107, 1968.

\bibitem[Agostinelli et~al.(2019)Agostinelli, McAleer, Shmakov, and
  Baldi]{agostinelli2019rubik}
Forest Agostinelli, Stephen McAleer, Alexander Shmakov, and Pierre Baldi.
\newblock Solving the rubik’s cube with deep reinforcement learning and
  search.
\newblock \emph{Nature Machine Intelligence}, 1, 07 2019.

\bibitem[Orseau et~al.(2018)Orseau, Lelis, Lattimore, and
  Weber]{orseau2018policy}
Laurent Orseau, Levi Lelis, Tor Lattimore, and Theophane Weber.
\newblock Single-agent policy tree search with guarantees.
\newblock In \emph{Advances in Neural Information Processing Systems 31}, pages
  3201--3211. Curran Associates, Inc., 2018.

\bibitem[Jabbari~Arfaee et~al.(2011)Jabbari~Arfaee, Zilles, and
  Holte]{jabbari2011bootstrap}
Shahab Jabbari~Arfaee, Sandra Zilles, and Robert~C. Holte.
\newblock Learning heuristic functions for large state spaces.
\newblock \emph{Artificial Intelligence}, 175\penalty0 (16):\penalty0
  2075--2098, 2011.

\bibitem[Doran et~al.(1966)Doran, Michie, and Kendall]{doran1966gbfs}
J.~E. Doran, D.~Michie, and David~George Kendall.
\newblock Experiments with the graph traverser program.
\newblock \emph{Proceedings of the Royal Society of London. Series A.
  Mathematical and Physical Sciences}, 294\penalty0 (1437):\penalty0 235--259,
  1966.

\bibitem[Pearl(1984)]{Pearl84}
Judea Pearl.
\newblock \emph{Heuristics - intelligent search strategies for computer problem
  solving}.
\newblock Addison-Wesley series in artificial intelligence. Addison-Wesley,
  1984.
\newblock ISBN 978-0-201-05594-8.

\bibitem[Silver et~al.(2016)Silver, Huang, Maddison, Guez, Sifre, Van
  Den~Driessche, Schrittwieser, Antonoglou, Panneershelvam, Lanctot,
  et~al.]{silver2016mastering}
David Silver, Aja Huang, Chris~J Maddison, Arthur Guez, Laurent Sifre, George
  Van Den~Driessche, Julian Schrittwieser, Ioannis Antonoglou, Veda
  Panneershelvam, Marc Lanctot, et~al.
\newblock Mastering the game of {G}o with deep neural networks and tree search.
\newblock \emph{Nature}, 529\penalty0 (7587):\penalty0 484--489, 2016.

\bibitem[Chaslot et~al.(2008)Chaslot, Winands, and van~den
  Herik]{chaslot2008parallel}
Guillaume M. J.~B. Chaslot, Mark H.~M. Winands, and H.~Jaap van~den Herik.
\newblock Parallel monte-carlo tree search.
\newblock In \emph{Computers and Games}, pages 60--71. Springer Berlin
  Heidelberg, 2008.

\bibitem[Pereira et~al.(2016)Pereira, Holte, Schaeffer, Buriol, and
  Ritt]{Pereira2016b}
Andr{\'e}~Grahl Pereira, Robert Holte, Jonathan Schaeffer, Luciana~Salete
  Buriol, and Marcus Ritt.
\newblock Improved heuristic and tie-breaking for optimally solving sokoban.
\newblock In \emph{International Joint Conference on Artificial Intelligence},
  2016.

\bibitem[Guez et~al.(2018)Guez, Mirza, Gregor, Kabra, Racaniere, Weber, Raposo,
  Santoro, Orseau, Eccles, Wayne, Silver, Lillicrap, and Valdes]{boxobanlevels}
Arthur Guez, Mehdi Mirza, Karol Gregor, Rishabh Kabra, Sebastien Racaniere,
  Theophane Weber, David Raposo, Adam Santoro, Laurent Orseau, Tom Eccles, Greg
  Wayne, David Silver, Timothy Lillicrap, and Victor Valdes.
\newblock An investigation of model-free planning: boxoban levels.
\newblock https://github.com/deepmind/boxoban-levels/, 14 Dec 2018, 2018.

\bibitem[Abel et~al.(2020)Abel, Bosboom, Coulombe, Demaine, Hamilton,
  Hesterberg, Kopinsky, Lynch, Rudoy, and Thielen]{abel2020witness}
Zachary Abel, Jeffrey Bosboom, Michael~J. Coulombe, Erik~D. Demaine, Linus
  Hamilton, Adam Hesterberg, Justin Kopinsky, Jayson Lynch, Mikhail Rudoy, and
  Clemens Thielen.
\newblock Who witnesses the witness? finding witnesses in the witness is hard
  and sometimes impossible.
\newblock \emph{Theor. Comput. Sci.}, 839:\penalty0 41--102, 2020.

\bibitem[Korf(1985)]{korf1985depth}
Richard~E. Korf.
\newblock Depth-first iterative-deepening.
\newblock \emph{Artificial Intelligence}, 27\penalty0 (1):\penalty0 97 -- 109,
  1985.

\bibitem[Felner et~al.(2004)Felner, Korf, and Hanan]{felner2004pdb}
A.~Felner, R.~E. Korf, and S.~Hanan.
\newblock Additive pattern database heuristics.
\newblock \emph{Journal of Artificial Intelligence Research}, 22:\penalty0
  279--318, 2004.

\bibitem[Ernandes and Gori(2004)]{ernandes2004heuristics}
Marco Ernandes and Marco Gori.
\newblock Likely-admissible and sub-symbolic heuristics.
\newblock In \emph{Proceedings of the 16th European Conference on Artificial
  Intelligence}, ECAI’04, pages 613--617. IOS Press, 2004.

\bibitem[Racani\`{e}re et~al.(2017)Racani\`{e}re, Weber, Reichert, Buesing,
  Guez, Jimenez~Rezende, Puigdom\`{e}nech~Badia, Vinyals, Heess, Li, Pascanu,
  Battaglia, Hassabis, Silver, and Wierstra]{racaniere2017imagination}
S\'{e}bastien Racani\`{e}re, Theophane Weber, David Reichert, Lars Buesing,
  Arthur Guez, Danilo Jimenez~Rezende, Adri\`{a} Puigdom\`{e}nech~Badia, Oriol
  Vinyals, Nicolas Heess, Yujia Li, Razvan Pascanu, Peter Battaglia, Demis
  Hassabis, David Silver, and Daan Wierstra.
\newblock Imagination-augmented agents for deep reinforcement learning.
\newblock In \emph{Advances in Neural Information Processing Systems 30}, pages
  5690--5701. Curran Associates, Inc., 2017.

\bibitem[Kingma and Ba(2015)]{kingma2015adam}
Diederik~P. Kingma and Jimmy Ba.
\newblock Adam: {A} method for stochastic optimization.
\newblock In Yoshua Bengio and Yann LeCun, editors, \emph{3rd International
  Conference on Learning Representations, {ICLR} 2015, San Diego, CA, USA, May
  7-9, 2015, Conference Track Proceedings}, 2015.
\newblock URL \url{http://arxiv.org/abs/1412.6980}.

\bibitem[Richter and Westphal(2010)]{Richter2010}
Silvia Richter and Matthias Westphal.
\newblock The lama planner: Guiding cost-based anytime planning with landmarks.
\newblock \emph{Journal of Artificial Intelligence Research}, 39\penalty0
  (1):\penalty0 127--177, 2010.

\bibitem[Helmert et~al.(2019)Helmert, Lattimore, Lelis, Orseau, and
  Sturtevant]{helmert2019iterative}
Malte Helmert, Tor Lattimore, Levi H.~S. Lelis, Laurent Orseau, and Nathan~R.
  Sturtevant.
\newblock Iterative budgeted exponential search.
\newblock In \emph{Proceedings of the Twenty-Eighth International Joint
  Conference on Artificial Intelligence, {IJCAI} 2019, Macao, China, August
  10-16, 2019}, pages 1249--1257. ijcai.org, 2019.

\end{thebibliography}

\ifarxiv
\clearpage
\begin{appendices}
\crefalias{section}{appendix}

\section{Table of Notation}\label{sec:notation_table}

A summary of the notation used in the paper can be found in \cref{tab:notation}
(on the last page for convenience).

\section{Machine Specifications}\label{sec:machine}

All experiments were run on 2.4 GHz CPUs with 8 GB of RAM (except 12 GB for Sokoban),
with Python 3.6, Tensorflow 2.0 with Keras. The training procedure was parallelized with 6 CPUs. Tests were performed on a single CPU. No GPU was used.

\section{Additional Implementation Remarks}

Some values ($\pol(n)$, $\flts(n)$) for \levints{} and \ouralg{} are maintained in log-space to avoid numerical issues.
For \ouralg{}, the heuristic value is enforced to be lower bounded by 0. 

For all algorithms (including PUCT in particular) no duplicate queries to the network are performed.

Ties for $f$ and $\flts$ are broken in favor of largest $g$-costs. However, with learned real-valued heuristics,
it is unlikely that two nodes have the same values.

\section{Neural Network Architectures}\label{sec:nnarchs}

The various NN architectures are as follow.
All convolution layers (Conv) have 32 filters of size 2x2 (ReLU) and are not padded.
All fully connected layers (FC) have 128 ReLU units.

\noindent
Policy-based algorithms (\levints{}):
\begin{align*}
    &\text{In}\rightarrow
    \text{Conv}\rightarrow\text{Conv}\rightarrow
    \text{FC}\rightarrow\text{4 linear + log-softmax}
\intertext{Heuristic-based algorithms (A*, WA*, GBFS):}
    &\text{In}\rightarrow
    \text{Conv}\rightarrow\text{Conv}\rightarrow
    \text{FC}\rightarrow\text{1 linear (heuristic)}
\intertext{Policy- and heuristic-based algorithms  (\ouralg, PUCT):}
    &\text{In}\rightarrow
    \text{Conv}\rightarrow\text{Conv}\rightarrow\begin{cases}
    \text{FC}\rightarrow\text{1 linear} &\text{(heuristic)} \\
    \text{FC}\rightarrow\text{4 linear} &\text{+ log-softmax}\\&\text{(policy)} \\
    \end{cases}
\end{align*}

We use an Adam optimization~\citep{kingma2015adam} with a step size of $10^{-4}$,
L2 regularization with constant $10^{-3}$.
We did not try any other hyperparameter values than the ones we report.

\subsection*{Network Inputs}

The input for Sokoban (\by{10}) is a one-hot encoding tensor with shape $10 \times 10 \times 4$
with the 4 features being: wall, avatar, boxes, and goal positions.

The input for The Witness (\by4) is a one-hot encoding tensor of shape $8 \times 8 \times 9$, where the 9 features
are the 4 bullet colors, the entrance, the exit, the separating line being drawn by the player, cells without bullets, and the tip of the line. We use images of size \by8 in our representation despite the problem being represented on a \by4 grid because the line occupies the spaces in-between the grid cells. 

The input for the sliding-tile puzzle (\by5) is a one-hot encoding tensor of shape $5 \times 5 \times 25$, with one image $5 \times 5$ for each tile.

\section*{Detailed Description of Domains}

\paragraph{Sokoban} Figure~\ref{fig:sokoban} shows one Sokoban problem of the Boxoban set~\citep{boxobanlevels}. In this game the goal is to make the avatar (shown in green) to \emph{push} (not pull) all boxes (shown in yellow) to their goal positions (shown in red).
For the state shown in the figure, 
there are four children nodes, one for moving up, down, left and right,
but moving down leads to the same state as the current one and thus a state pruning
may be performed to avoid unnecessary node expansions.
Similarly, the policy can also assign a probability 0 to moving down;
but note that the heuristic may not be able to assign a large number to repeating states
if the heuristic value depends only on the current state.

\begin{figure}[t]
    \centering
    \includegraphics[height=100px]{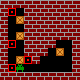}
    \caption{A state of a Boxoban problem.}
    \label{fig:sokoban}
\end{figure}

\paragraph{The Witness}
We use the puzzle from the commercial game The Witness where one needs to find a path connecting the bottom left corner of a \by{4} grid to an \emph{exit point}. This path must separate the cells in the grid into regions with cells of the same color---cells without bullets do not matter.
Figure~\ref{fig:witness} shows one such problem. The cells of the puzzles we used in our experiments could be either empty or contain a bullet with one of four colors. The path from the bottom left corner to the exit point cannot go more than once over a black dot marking the intersection of the cells.
\begin{figure}[t]
    \centering
    \includegraphics[height=0.2\textheight]{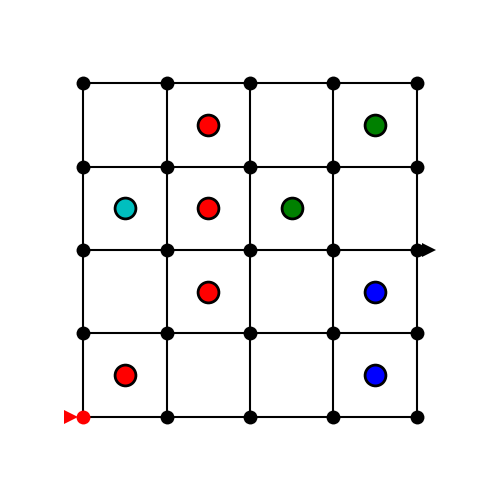}
    \caption{The initial state of a \by{4} witness puzzle problem with four different colors.
    The entrance is at the bottom-left and the exit at the center-right.}
    \label{fig:witness}
\end{figure}

\paragraph{Sliding-Tile Puzzle}
On a grid of $\by5$ cells, there is one empty cell and 24 tiles, labelled from 1 to 24.
The only possible moves are to slide one tile onto the empty cell, which means there are always between 2 and 4 possible moves.
The goal is to slides the tiles until reaching the goal state as depicted on the left-hand-side of Figure~\ref{fig:24-puzzle}.
 The initial state of the problems we used are scrambled and the goal is to find how to restore the configuration shown on the left-hand side of the figure.

\newenvironment{stpmatrix}{
\begin{tikzpicture}[scale=0.52,every node/.style={scale=0.52}]
\draw[xstep=1cm,ystep=1cm,color=gray] (0,0) grid (5,5);
\matrix[matrix of nodes,
inner sep=0pt,
anchor=south west,
nodes={inner sep=0pt,text width=1cm,align=center,minimum height=1cm},
]
}{
\end{tikzpicture}
}

\begin{figure}[h!]
\centering
\begin{stpmatrix}
{   & 1  & 2  & 3  & 4  \\
5  & 6  & 7  & 8  & 9  \\
10 & 11 & 12 & 13 & 14 \\
15 & 16 & 17 & 18 & 19 \\
20 & 21 & 22 & 23 & 24 \\
};
\end{stpmatrix}
\begin{stpmatrix}
{1  & 6  & 2  & 3  & 4  \\
5  &    & 7  & 8  & 9  \\
10 & 11 & 12 & 13 & 14 \\
15 & 16 & 17 & 18 & 19 \\
20 & 21 & 22 & 23 & 24 \\
};
\end{stpmatrix}
\begin{stpmatrix}
{16 & 8  & 22 & 14 & 5  \\
20 & 23 & 9  & 17 &    \\
3  & 4  & 10 & 1  & 6  \\
2  & 24 & 12 & 11 & 18 \\
7  & 15 & 21 & 19 & 13 \\
};
\end{stpmatrix}
\caption{ Sliding tile puzzle (\by5).
Left: The goal state. Middle: Two moves away from the goal state. Right: A test problem.}
\label{fig:24-puzzle}
\end{figure}

\section{Learning Curves}\label{sec:learning_curves}

\Cref{fig:more_learning} shows the number of remaining problems for each domain over the Bootstrap iterations,
for 5 different training runs of each algorithm.
All runs lie within the shadowed areas.
The plotted lines correspond to the runs with the fewest remaining unsolved problems and the end of training, which correspond to the networks that are used for testing.
Note that solving more problems means spending more time doing optimization,
which is why the number of expansions is shorter for algorithms that learn more.

\paragraph{Sokoban}
Algorithms that use a policy (\levints{}, \ouralgh{}, \ouralghat{}), except PUCT, 
solve more problems than the heuristic-based algorithms, but combining
both types of information in \ouralgh{} and \ouralghat{} gives the best results.

\paragraph{The Witness}
This domain appears to be easier to learn for policy-driven algorithms.
In particular, PUCT shines compared to A*, WA* and GBFS, but \ouralgh{}, \ouralghat{} and in particular \levints{} learn even faster.

\paragraph{Sliding Tile Puzzle}
On this domain, GBFS solves at most 4 problems and cannot get off the ground during training. 
This appears to be because GBFS with a random heuristic does not perform a systematic search like A* for example.
PUCT also performs poorly, but still makes steady progress, solving about 40 more problems per Bootstrap iteration, for a total of less than 500 solved problems after exhausting the time budget.
This is likely because PUCT has very long running times in particular for long solutions, as is the case in this domain.
Even though \levints{} and \ouralgh{} solve fewer problems than A*, the trend at the end
suggests that they would soon surpass A* after some more training.
In this domain, the policy does not seem very helpful, and only learning a strong heuristic appears to be working.

\section{More Test Results}\label{sec:best_test_results}

\subsubsection*{Test profiles}

\Cref{fig:profiles} shows the number of expansions (on log scale) per problem, when problems are sorted
from easy to hard for each individual solver.
The networks used are the ones with the best training scores.
On the sliding tile puzzle, PUCT did not solve any problem,
and A* and \levints{} barely solved a few problems.
We also report the profiles of PUCT and \ouralgh{} when using the neural network trained
by \ouralghat{} (dashed lines). This shows a substantial improvement for both algorithms when
supplied with a good neural network, \ouralgh{}'s performance almost matching that of \ouralghat.
Similarly on Sokoban, PUCT vastly benefits from using the network trained with \ouralghat,
but it is still far less efficient than \ouralghat{} with the same network,
and solves less than 800 problems---while being allocated more time than all the other algorithms.

\subsubsection*{Best test results}
In \cref{tab:best_test_results} we report the \emph{best test} results
over the 5 trained networks per algorithm, along with a few more parameters
values for WA* and PUCT.
The results from the domain-independent LAMA planner~\citep{Richter2010} on Sokoban are as reported by \citep{orseau2018policy}.

\subsubsection*{Tests with a fixed time limit per problem}

As for training, we used Bootstrap during testing. 
This is a little unusual and it is more traditional to use a fixed time limit
per problem.
Both settings have issues.
With a fixed time limit,
a rank reversal between solvers can happen by increasing the time limit.
With Bootstrap, such a rank reversal can happen by increasing the total time
or by re-ordering the problems.
However, in both schemes, these issues do not happen
if the solvers solve all problems within the allocated time.

Using the Bootstrap process during testing ensures that algorithms that solve most problems quickly but need a long time to solve a few problems can try hard to solve all problems.
But comparing with a fixed time limit can be informative too.
However, for 1000 problems, two days of computation (as used with Bootstrap) correspond to a fixed limit of less than 3 minutes per problem.
Clearly, using such a low limit will neither provide meaningful results for slower algorithms, nor allow faster algorithms to tackle the few problems where they may take more than 3 minutes.
So instead we compare the algorithms on 100 problems with a time limit of 30 minutes per problem.
The results are shown in \cref{tab:fixed_time}.
As can be observed, the results for the best algorithms are consistent with the results using Bootstrap during testing.

\subsubsection*{Comparing algorithms with fixed parameters}

Although this paper is more interested with the question of how a search algorithm helps the learning process, it is also interesting to consider
how the solvers compare when using a fixed policy and a fixed heuristic function.

To be fair and accurate, several different policies and heuristic functions should be considered.
But to get a first insight into this question, we can use the network trained by \ouralghat{}---which is the only network with both a policy head and a heuristic head with good results.

Results are reported \Cref{tab:fixed}.
The caveat is of course that because it was trained with PHS*, it may not be so surprising that PHS* obtains the best results.
Nevertheless, these results show that combining the heuristic function with the policy gives better results than using each independently.

\begin{table}[tb]
\begin{center}
\begin{tabular}{lrrrr}
\toprule
\multicolumn{1}{c}{Alg.}& \multicolumn{1}{c}{Solved}  & \multicolumn{1}{c}{Length} & \multicolumn{1}{c}{Expansions} & \multicolumn{1}{c}{Time(s)} \\   
\midrule
\multicolumn{5}{c}{Sokoban \by{10} (test)} \\   
\midrule
GBFS         &    856 & \badnum{36.0} &  \atleast{7096.0} & \atleast{52.4} \\
PUCT, c=1    &    902 & \badnum{42.6} & \atleast{11329.8} & \atleast{83.8} \\
A*           &   \winner{1000} &          \winner{31.1} &            9043.5 &           51.8 \\
WA*, w=1.5   &   \winner{1000} &          33.6 &            5461.8 &           36.8 \\
\levints{}   &   \winner{1000} &          38.1 &            1760.5 &           11.9 \\
\ouralgh{}   &   \winner{1000} &          38.0 &            1696.2 &           \winner{10.6} \\
\ouralghat{} &   \winner{1000} &          37.6 &            \winner{1522.1} &           11.5 \\
\midrule
\multicolumn{5}{c}{The Witness \by4 (test)} \\   
\midrule
GBFS         &    444 & \badnum{15.4} & \atleast{23986.0} & \atleast{95.5} \\
WA*, w=1.5   &    927 & \badnum{14.1} & \atleast{19411.9} & \atleast{90.0} \\
A*           &   \winner{1000} & \winner{14.2} &           20390.4 &           97.7 \\
PUCT, c=1    &   \winner{1000} &          15.0 &            3798.8 &           21.0 \\
\levints{}   &   \winner{1000} &          14.9 &             903.4 &            4.9 \\
\ouralgh{}   &   \winner{1000} &          14.9 &             869.4 &            4.7 \\
\ouralghat{} &   \winner{1000} &          15.0 &    \winner{781.5} & \winner{4.1} \\
\midrule
\multicolumn{5}{c}{Sliding Tile Puzzle \by5 (test)} \\   
\midrule
PUCT, c=1    &    786 & \badnum{220.9} &  \atleast{30722.6} & \atleast{115.7} \\
A*           &    835 & \badnum{153.3} & \atleast{112369.6} &  \atleast{81.4} \\
WA*, w=1.5   &    916 & \badnum{175.4} & \atleast{117187.5} &  \atleast{80.8} \\
GBFS         &    998 &          390.2 &  \atleast{72444.5} &  \atleast{50.0} \\
\levints{}   &   \winner{1000} &          233.9 &             3175.5 &             \winner{2.4} \\
\ouralgh{}   &   \winner{1000} &          231.1 &             3026.7 &             \winner{2.5} \\
\ouralghat{} &   \winner{1000} & \winner{224.0} &    \winner{2867.2} &             \winner{2.5} \\
\bottomrule
\end{tabular}
\end{center}
\caption{
Test performance of the search algorithms when using the neural network
trained with PHS* with the best training performance.
Lengths are solution lengths, and the last three columns are averaged over solved problems only;
hence results like \badnum{123} cannot be properly compared with.
$^\dagger$Reported from \citet{orseau2018policy}.}
\label{tab:fixed}
\end{table}

\begin{table}[tb]
\begin{center}
\begin{tabular}{lrrrr}
\toprule
\multicolumn{1}{c}{Alg.}& \multicolumn{1}{c}{Solved}  & \multicolumn{1}{c}{Length} & \multicolumn{1}{c}{Expansions} & \multicolumn{1}{c}{Time(s)} \\   
\midrule
\multicolumn{5}{c}{Sokoban \by{10} (test)} \\   
\midrule
PUCT, c=2   &     154 &      \badnum{22.8} &     \atleast{9\,244.7} & \atleast{34.7} \\
PUCT, c=1.5 &     171 &      \badnum{22.4} &     \atleast{9\,664.0} & \atleast{36.2} \\
PUCT, c=1   &     229 &      \badnum{24.9} &    \atleast{10\,021.3} & \atleast{39.7} \\
GBFS        &     945 &      \badnum{37.7} &     \atleast{7\,284.2} & \atleast{49.2} \\
A*          &     999 &      \winner{32.9} &     \atleast{9\,487.3} & \atleast{67.4} \\
WA*, w=2.5  &    \winner{1\,000} &      36.2 &     4\,754.8 & 32.9 \\
WA*, w=2    &    \winner{1\,000} &      35.6 &     3\,298.5 & 22.8 \\
LAMA$^\dagger$ &    \winner{1\,000} &      51.6 &    3\,151.3 & \badnum{N/A} \\
WA*, w=1.5  &    \winner{1\,000} &      34.6 &     3\,092.6 & 21.5 \\
\levints{}       &    \winner{1\,000} &      40.1 &     2\,640.4 & 19.5 \\
\ouralgh{}  &    \winner{1\,000} &      38.9 &     1\,962.2 & 14.8 \\
\ouralghat{}&    \winner{1\,000} &      37.6 &     \winner{1\,522.1} & \winner{11.3} \\
\midrule
\multicolumn{5}{c}{The Witness \by4 (test)} \\   
\midrule
GBFS        &     290 &      \badnum{13.3} &    \atleast{10\,127.9} & \atleast{44.6} \\
WA*, w=3    &     673 &      \badnum{13.8} &    \atleast{12\,195.6} & \atleast{48.9} \\
WA*, w=2.5  &     674 &      \badnum{14.0} &    \atleast{13\,988.0} & \atleast{51.9} \\
WA*, w=2    &     835 &      \badnum{14.2} &    \atleast{14\,305.2} & \atleast{55.5} \\
A*          &     975 &      \badnum{14.1} &    \atleast{13\,058.9} & \atleast{65.5} \\
WA*, w=1.5  &     999 &      14.6 &    18\,345.2 & 71.5 \\
PUCT, c=1.5 &    \winner{1\,000} &      15.4 &     3\,639.7 & 22.7 \\
PUCT, c=1   &    \winner{1\,000} &      15.3 &     3\,602.7 & 20.8 \\
PUCT, c=2   &    \winner{1\,000} &      15.4 &     2\,890.6 & 18.2 \\
\ouralgh{}  &    \winner{1\,000} &      14.6 &      221.5 &  \winner{1.8} \\
\levints{}       &    \winner{1\,000} &      14.8 &      219.7 &  \winner{1.6} \\
\ouralghat{}&    \winner{1\,000} &      \winner{14.4} &      \winner{190.5} &  \winner{1.7} \\
\midrule
\multicolumn{5}{c}{Sliding Tile Puzzle \by5 (test)} \\   
\midrule
GBFS        &       0 &        ---&         ---&   ---\\
PUCT, c=1   &       0 &        ---&         ---&   ---\\
PUCT, c=1.5 &       0 &        ---&         ---&   ---\\
PUCT, c=2   &       0 &        ---&         ---&   ---\\
A*          &       3 &      \badnum{87.3} &    \atleast{34\,146.3} & \atleast{27.2} \\
\ouralgh{}  &       4 &     \badnum{119.5} &    \atleast{58\,692.0} & \atleast{55.3} \\
\levints{}       &      30 &     \badnum{159.6} &    \atleast{65\,544.6} & \atleast{56.7} \\
\ouralghat{}&    \winner{1\,000} &     222.8 &     2\,764.0 &  3.0 \\
WA*, w=1.5  &    \winner{1\,000} &     \winner{127.3} &     1\,860.5 &  \winner{1.4} \\
WA*, w=3    &    \winner{1\,000} &     134.3 &     1\,840.0 &  \winner{1.4} \\
WA*, w=2    &    \winner{1\,000} &     130.3 &     1\,801.6 &  \winner{1.5} \\
WA*, w=2.5  &    \winner{1\,000} &     133.6 &     \winner{1\,793.1} &  \winner{1.4} \\
\bottomrule
\end{tabular}
\end{center}
\caption{Best test results among the 5 trained networks for each algorithm.
Lengths are solution lengths, and the last three columns are averaged over solved problems only;
hence results like \badnum{123} cannot be properly compared with.
$^\dagger$Reported from \citet{orseau2018policy}.}
\label{tab:best_test_results}
\end{table}

\begin{table}[tb]
\begin{center}
\begin{tabular}{lrrrr}
\toprule
\multicolumn{1}{c}{Alg.}& \multicolumn{1}{c}{Solved}  & \multicolumn{1}{c}{Length} & \multicolumn{1}{c}{Expansions} & \multicolumn{1}{c}{Time(s)} \\   
\midrule
\multicolumn{5}{c}{Sokoban \by{10} (test)} \\   
\midrule
PUCT, c=1    &      41 &      \badnum{38.2} &    \badnum{60\,012.5} & \badnum{425.0} \\
GBFS         &      99 &      37.7 &     \badnum{9\,841.0} &  \badnum{63.9} \\
A*           & \winner{100} & \winner{32.5} & 7\,064.4            & 49.4 \\
WA*, w=1.5   & \winner{100} & 34.1          & 1\,863.5            & 13.7 \\
\levints{}   & \winner{100} & 40.0          & 1\,392.5            & 10.5 \\
\ouralgh{}   & \winner{100} & 39.5          & 848.2             & 6.7 \\
\ouralghat{} & \winner{100} & 37.3          & \winner{717.3}    & \winner{5.6} \\
\midrule
\multicolumn{5}{c}{The Witness \by4 (test)} \\   
\midrule
GBFS         & \winner{100} & 14.7          & 44\,428.2        & 135.7 \\
WA*, w=1.5   & \winner{100} & 13.8          & 11\,503.8        & 39.9 \\
A*           & \winner{100} & \winner{13.6} & 11\,337.7        & 52.0 \\
PUCT, c=1    & \winner{100} & 14.6          & 3\,075.3         & 20.8 \\
\ouralgh{}   & \winner{100} & 13.8          & 193.6          & \winner{1.3} \\
\levints{}   & \winner{100} & 13.9          & 187.1          & \winner{1.2} \\
\ouralghat{} & \winner{100} & 13.7          & \winner{181.1} & \winner{1.4} \\
\midrule
\multicolumn{5}{c}{Sliding Tile Puzzle \by5 (test)} \\   
\midrule
GBFS         & 0            & ---            & ---                & --- \\
PUCT, c=1    & 0            & ---            & ---                & --- \\
A*           & 2            & \badnum{87.0}  & \badnum{897\,140.5}  & \badnum{440.8} \\
\ouralgh{}   & 14           & \badnum{155.4} & \badnum{1\,314\,654.1} & \badnum{907.0} \\
\levints{}   & 22           & \badnum{154.5} & \badnum{1\,074\,269.4} & \badnum{676.9} \\
\ouralghat{} & \winner{100} & 225.2          & 2\,883.0             & 1.7 \\
WA*, w=1.5   & \winner{100} & \winner{128.9} & \winner{1\,957.8}    & \winner{0.9} \\
\bottomrule
\end{tabular}
\end{center}
\caption{Test results for the networks with the best training scores for each algorithm,
with a fixed time limit of 30 minutes per problem.
Lengths are the average solution lengths. The last three columns (Length, Expansions, and Time) are averaged over solved problems only;
hence results like \badnum{123} cannot be properly compared with.}
\label{tab:fixed_time}
\end{table}

\section{\ouralg{} with Safe State Pruning}\label{sec:safepruning}

We say that a \emph{state pruning} is performed when the \code{continue} instruction is triggered
in \cref{alg:phs}.

The theorems of the main text hold for \cref{alg:phs} when $\state(n)=n$ for all $n$, i.e., no state pruning ever happens.
If the latter condition is not true, it may happen that a solution node below a node that was pruned has a
$\fltsmon$-value lower than any other solution node, and pruning prevents the algorithm from finding and returning it.
But when many nodes represent the same states, it may be inefficient to
not perform state pruning.

In \cref{alg:phs_reexp}, we propose a variant that performs safe state pruning
if some assumptions (defined below) are satisfied,
thus avoiding following paths that provably lead to 
solutions with larger $\fltsmon$-values. 

\begin{algorithm}[htb!]
\begin{lstlisting}
def (*\ouralg{}*)_safe_state_pruning($\rootnode$):
  q = priority_queue(order_by=$\flts$)
  q.insert($\rootnode$)
  visited_states = {}
  while q is not empty:
    n = q.extract_min() # node of min $\flts$-value
    s = state(n)
    ($\flts_s$, $\pol_s$) = visited_states.get(s, 
                            default=($\infty$, 0))
    if $\flts_s \leq \flts($n$)$ and $\pol_s \geq \pol($n$)$:
      # Can safely prune the search here
      continue  # state pruning
    if $\pol_s \leq \pol($n$)$:
      visited_states.set(s, ($\flts($n$), \pol($n$)$))
    incur_loss $\llts($n$)$
    if is_solution(n):
      return n
    # Node expansion
    for n$'$ in children(n):
      q.insert(n$'$, $\flts$(n$'$))
  return False
\end{lstlisting}
\caption{The \ouralglong{} algorithm (\ouralg) with safe state pruning.
Assumes that $\pol(\cdot|n)$ and $\jlts(n)$ depend only on $\state(n)$.}
\label{alg:phs_reexp}
\end{algorithm}

For simplicity of the formalism but without loss of generality, we assume that states are `canonical' nodes,
that is, $\stateset \subseteq \nodeset$,
and the function $\state(n)$ maps a node $n$ to its canonical node $n'$.
Though, in practice, a state can be anything such as a string or an integer.

\begin{assumptions}\label{ass:safe_pruning}
a) Most quantities depend only on the state:
    For all nodes $n'\in\children(n)$,
    \begin{align*}
     \llts(n) &= \llts(\state(n)), \\
     \issolution(n) &= \issolution(\state(n)), \\
     \{\state(n'): n'\in\children(n)\} &= \children(\state(n)), \\
    \forall n'\in\children(n'): \pol(n'|n) &= \pol(\state(n') | \state(n)), \\
     \text{and }\jlts(n) &= \jlts(\state(n)),
    \end{align*} 
    and, b) \ouralg{} is defined as in \cref{alg:phs_reexp}.
\end{assumptions}

Some remarks:
\begin{itemize}
    \item These assumptions are met in all three domains and neural networks tested in the main text.
    \item The path loss is still dependent on the path, and thus in general
$\glts(n) \neq \glts(\state(n))$.
    \item Taking $\state(n)=n$ satisfies all of \cref{ass:safe_pruning}a.
\end{itemize}

Recall that $\fltsmon(n) = \max_{n'\in\ancn(n)}\flts(n)$
and that if $\fltsmon(n_1) < \fltsmon(n_2)$ then \ouralg{}
extracts $n_1$ from the priority queue before $n_2$, 
unless an ancestor of $n_1$ has been pruned.

\begin{lemma}[Safe state pruning]\label{lem:safe_pruning}
Under \cref{ass:safe_pruning},
let any two nodes $n_1$ and $n_2$ such that $\state(n_1)=\state(n_2)$
and $n_1$ is expanded before $n_2$ is extracted from the priority queue.
Let $n_1'\in\children(n_1)$ and $n_2'\in\children(n_2)$
such that $\state(n_1')=\state(n_2')$.
Then if $\flts(n_1) \leq \flts(n_2)$ and $\pol(n_1)\geq \pol(n_2)$,
for any solution node $n^*_2 \in \goalset\cap\descn(n_2)$
there exists a solution node $n^*_1\in\goalset\cap\descn(n_1)$
such that $\fltsmon(n^*_1) \leq \fltsmon(n^*_2)$.
\end{lemma}
\begin{proof}

First, observe that, as for BFS, since $n_1$ is expanded before $n_2$ is extracted from the queue then necessarily
$\fltsmon(n_1)\leq \fltsmon(n_2)$.
Then,
\begin{align*}\resetalph
                    \flts(n_1) &\alphoverset{\leq} \flts(n_2) \\
    \Leftrightarrow \frac{\glts(n_1)}{\pol(n_1)} &\alphoverset{\leq} \frac{\glts(n_2)}{\pol(n_2)} \\
    \Leftrightarrow \frac{\glts(n_1)+\llts(n_1')}{\pol(n_1)} &\alphoverset{\leq} \frac{\glts(n_2) + \llts(n_2')}{\pol(n_2)} \\
    \Leftrightarrow \jlts(n_1')\frac{\glts(n_1)+\llts(n_1')}{\pol(n_1)\pol(n_1'|n_1)} &\alphoverset{\leq} \jlts(n_2')\frac{\glts(n_2) + \llts(n_2')}{\pol(n_2)\pol(n_2'|n_2)} \\
    \Leftrightarrow \flts(n_1') &\alphoverset{\leq} \flts(n_2') \\
    \Rightarrow 
    \max\{\fltsmon(n_1), \flts(n_1')\}
    &\leq \max\{\fltsmon(n_1), \flts(n_2')\}
    \\
    \Leftrightarrow
    \fltsmon(n_1') &\alphoverset{\leq} \fltsmon(n_2')
\end{align*}\resetalph
where
\alphnextref{} by assumption,
\alphnextref{} since $\jlts(n_1)=\jlts(n_2)$,
\alphnextref{} using $\llts(n_1')=\llts(n_2')$ and $1/\pol(n_1) \leq 1/\pol(n_2)$,
\alphnextref{} using $\pol(n_1'|n_1) = \pol(n_2'|n_2)$ and $\jlts(n_1')=\jlts(n_2')$
\alphnextref{} by definition of $\flts$,
\alphnextref{} using $\fltsmon(n_1) \leq \fltsmon(n_2)$ and by definition of $\fltsmon$.

Furthermore, since $\pol(n_1'|n_1) = \pol(n_2'|n_2)$ we also have $\pol(n_1') \geq \pol(n_2')$,
and therefore the result holds by induction.
\end{proof}

\Cref{lem:safe_pruning} means that we can safely prune $n_2$
and that \cref{alg:phs_reexp} will still return a node in $\argmin_{n^*\in\goalset} \fltsmon(n^*)$.
We can now prove that performing state pruning as in \cref{alg:phs_reexp} does not harm the bound of 
\cref{thm:upper_bound}:

\begin{theorem}[\ouralg{} bound with safe state pruning]\label{thm:upper_bound_pruning}
Under \cref{ass:safe_pruning}, the statement of \cref{thm:upper_bound} holds.
\end{theorem}
\begin{proof}
Follows from \cref{lem:safe_pruning}, and the rest is as in the proof of \cref{thm:upper_bound}.
\end{proof}

If many states are pruned during the search, then this can only improve $\Llts(\ouralg, .)$
over \cref{thm:upper_bound_pruning}, that is, the bound displayed in \cref{thm:upper_bound}
can become loose.

It can also happen that many state prunings are missed due to the necessity of performing only \emph{safe} pruning.
In such cases, the bound of \cref{thm:upper_bound_pruning} still holds, but the algorithm 
could be more efficient.
For such cases one can use the IBEX framework~\citep{helmert2019iterative}
which iteratively increases the cost bound (here, we would choose $\log \fltsmon(\cdot)$)
at the price of a log-cost factor compared to the ideal algorithm that orders the search optimally so as to perform as many state prunings as possible.

\begin{figure*}[t]
    \centering
    \includegraphics[height=0.3\textheight]{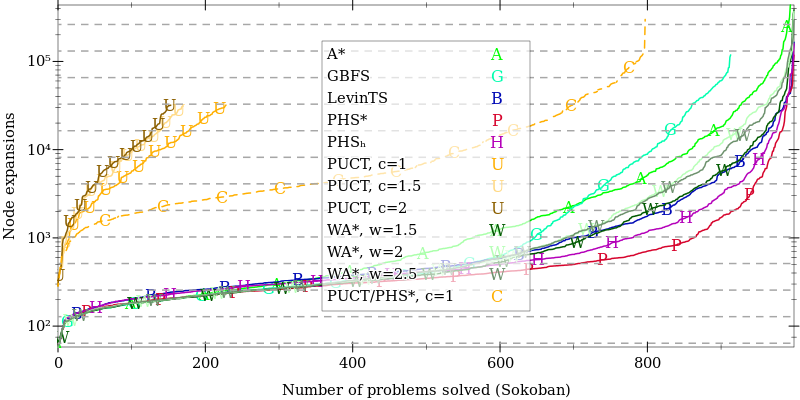}
    \includegraphics[height=0.3\textheight]{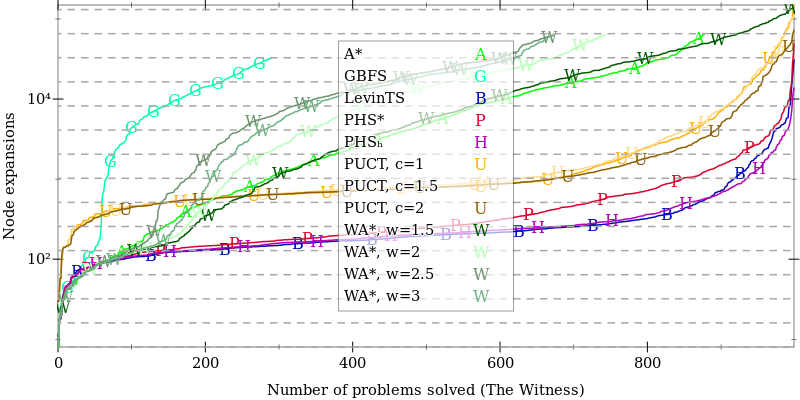}
    \includegraphics[height=0.3\textheight]{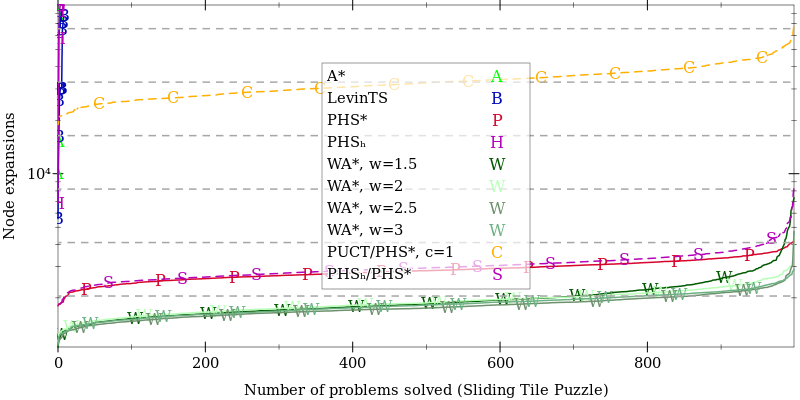}
    \caption{Test profiles: Number of node expansions (log scale) per problem
    on the test set, for the networks with the best training results for each algorithm.
    Horizontal dashed lines are the powers of 2.
    A dashed line with label A/B is for the search algorithm A that uses
    the neural network trained with algorithm B.}
    \label{fig:profiles}
\end{figure*}

\begin{figure*}[t]
    \centering
    \includegraphics[width=0.49\textwidth]{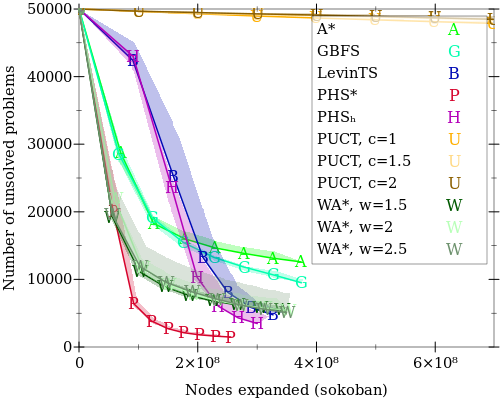}
    \includegraphics[width=0.49\textwidth]{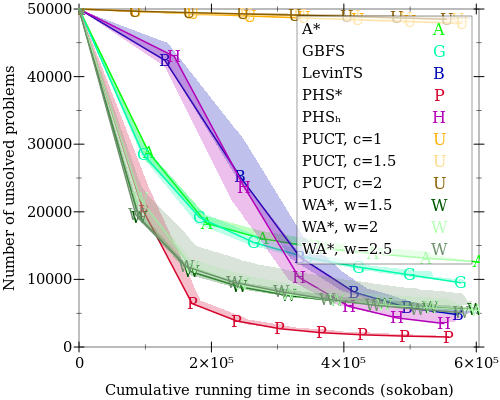}
    \includegraphics[width=0.49\textwidth]{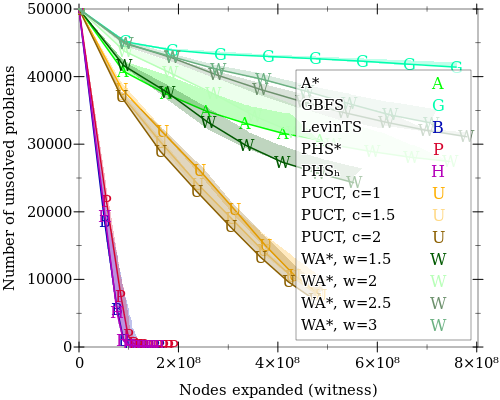}
    \includegraphics[width=0.49\textwidth]{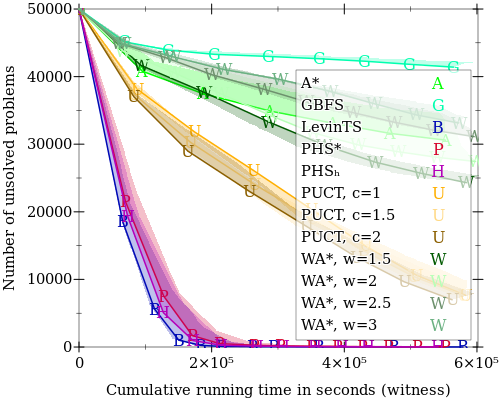}
    \includegraphics[width=0.49\textwidth]{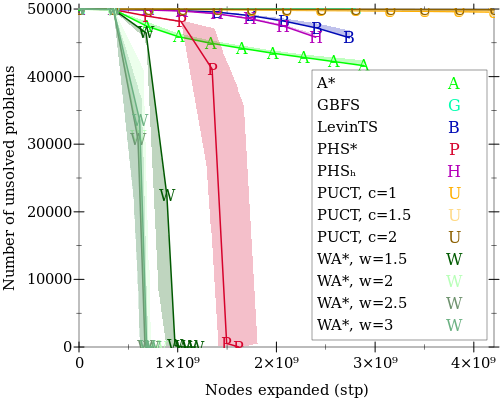}
    \includegraphics[width=0.49\textwidth]{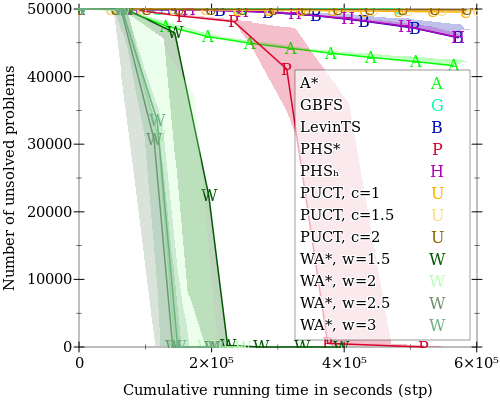}
    \caption{Number of unsolved problems compared to the cumulative number of nodes
    expanded (left) or cumulative time (right) over the Bootstrap iterations.
    Lines correspond to the runs with the least remaining unsolved problems.
    All 5 training runs per algorithm lie within the colored areas.}
    \label{fig:more_learning}
\end{figure*}

\clearpage

\begin{table}[b]
    \centering
    \begin{tabular}{l|l}
        Symbol & Meaning \\
        \hline
$\rootnode$   & The root of the tree                            \\
$n$           & A node in the tree                              \\
$n^*$         & A solution node or some target node             \\
$\parent(n)$  & Parent node of $n$                              \\
$\ancn(n)$    & Ancestors of $n$ and $n$ itself                 \\
$\descn(n)$   & Descendants of $n$ and $n$ itself               \\ 
$d(n)$        & Depth of node $n$\ \ ($d(\rootnode)=0$)         \\
$d_0(n)$      & $=d(n)+1$                                       \\
$\nodeset$    & Set of nodes                                    \\
$\llts(n)$    & Instantaneous loss function at a node $n$       \\
$\glts(n)$    & Path loss from the root to a node $n$           \\
$\Llts(S, n)$ & Search loss of algorithm $S$ when reaching $n$  \\
$\flts(\cdot)$    & Evaluation function of \ouralg{}            \\
$\fltsmon(\cdot)$ & Like $\flts$, but made monotone non-decreasing      \\
                  & from parent to child                        \\
$h(n)$        & Estimate of the $\glts$-cost of a solution node \\
              & below $n$ (A*-like heuristic)                   \\
$\jlts(\cdot)$    & Heuristic factor                            \\
$\jltsmon(\cdot)$ & Heuristic factor corresponding to $\fltsmon$\\ 
$\jlts_h(\cdot)$  & $\jlts$ that uses an A*-like heuristic $h$; \\
              & \ouralg-admissible if $h$ is A*-admissible      \\
$\jltshat_h(\cdot)$  & Like $\jlts_h$ but aggressively rescaled \\
\ouralg       & \ouralglong                                     \\
\ouralgh      & \ouralg{} where $\jlts=\jlts_h$ for given $h$   \\
\ouralghat    & \ouralg{} where $\jlts=\jltshat_h$ for given $h$\\
    \end{tabular}
    \caption{Notation used in the paper.}
    \label{tab:notation}
\end{table}

\end{appendices}

\fi

\end{document}